\documentclass{article}

\usepackage{microtype}
\usepackage{graphicx}
\usepackage{subfigure}
\usepackage{booktabs} % for professional tables

\usepackage{hyperref}

\usepackage[accepted]{icml2024}

\usepackage{amsmath}
\usepackage{amssymb}
\usepackage{mathtools}
\usepackage{amsthm}
\usepackage{bm}
\usepackage{amsfonts}
\usepackage{float}
\usepackage[normalem]{ulem}
\usepackage{tabularx}
\usepackage{caption}

\newcolumntype{L}{>{\raggedright\arraybackslash}X}

\usepackage[capitalize,noabbrev]{cleveref}

\theoremstyle{plain}
\newtheorem{theorem}{Theorem}[section]

\newtheorem{lemma}[theorem]{Lemma}

\theoremstyle{definition}

\theoremstyle{remark}

\usepackage[textsize=tiny]{todonotes}

\icmltitlerunning{Nearest Neighbour Score Estimators for Diffusion Generative Models}

\begin{document}

% Custom commands
\newcommand{\bx}{\mathbf{x}}  % Bold x.
\newcommand{\bz}{\mathbf{z}}  % Bold z.
\newcommand{\student}{\bm{f}_{\theta}}
\newcommand{\teacher}{\bm{f}_{\theta^-}}
\newcommand{\bxi}{\bx^{(i)}}
\newcommand{\bxj}{\bx^{(j)}}
\newcommand{\bxk}{\bx^{(k)}}
\newcommand{\dataset}{\mathcal{D}}
\newcommand{\pmean}{\mathbb{E}\left[ \bx | \bz, t \right]}
\newcommand{\snispmean}{\widehat{\pmean}_{KNN}}
\newcommand{\pdata}{p_{data}\left(\bx\right)}

\twocolumn[
\icmltitle{Nearest Neighbour Score Estimators for Diffusion Generative Models}

% It is OKAY to include author information, even for blind
% submissions: the style file will automatically remove it for you
% unless you've provided the [accepted] option to the icml2024
% package.

% List of affiliations: The first argument should be a (short)
% identifier you will use later to specify author affiliations
% Academic affiliations should list Department, University, City, Region, Country
% Industry affiliations should list Company, City, Region, Country

% You can specify symbols, otherwise they are numbered in order.
% Ideally, you should not use this facility. Affiliations will be numbered
% in order of appearance and this is the preferred way.
\icmlsetsymbol{equal}{*}

\begin{icmlauthorlist}
\icmlauthor{Matthew Niedoba}{ubc,iai}
\icmlauthor{Dylan Green}{ubc,iai}
\icmlauthor{Saeid Naderiparizi}{ubc,iai}
\icmlauthor{Vasileios Lioutas}{ubc,iai}
\icmlauthor{Jonathan Wilder Lavington}{ubc,iai}
\icmlauthor{Xiaoxuan Liang}{ubc,iai}
\icmlauthor{Yunpeng Liu}{ubc,iai}
\icmlauthor{Ke Zhang}{ubc,iai}
\icmlauthor{Setareh Dabiri}{iai}
\icmlauthor{Adam Ścibior}{ubc,iai}
\icmlauthor{Berend Zwartsenberg}{iai}
\icmlauthor{Frank Wood}{ubc,iai}

% \icmlauthor{Firstname2 Lastname2}{equal,yyy,comp}
% \icmlauthor{Firstname3 Lastname3}{comp}
% \icmlauthor{Firstname4 Lastname4}{sch}
% \icmlauthor{Firstname5 Lastname5}{yyy}
% \icmlauthor{Firstname6 Lastname6}{sch,yyy,comp}
% \icmlauthor{Firstname7 Lastname7}{comp}
% %\icmlauthor{}{sch}
% \icmlauthor{Firstname8 Lastname8}{sch}
% \icmlauthor{Firstname8 Lastname8}{yyy,comp}
%\icmlauthor{}{sch}
%\icmlauthor{}{sch}
\end{icmlauthorlist}

\icmlaffiliation{ubc}{Department of Computer Science, University of British Columbia, Vancouver, Canada}
\icmlaffiliation{iai}{Inverted AI, Vancouver, Canada}
% \icmlaffiliation{sch}{School of ZZZ, Institute of WWW, Location, Country}

\icmlcorrespondingauthor{Matthew Niedoba}{mniedoba@cs.ubc.ca}
% \icmlcorrespondingauthor{Firstname2 Lastname2}{first2.last2@www.uk}

% You may provide any keywords that you
% find helpful for describing your paper; these are used to populate
% the "keywords" metadata in the PDF but will not be shown in the document
\icmlkeywords{Machine Learning, ICML}

\vskip 0.3in
]

% this must go after the closing bracket ] following \twocolumn[ ...

% This command actually creates the footnote in the first column
% listing the affiliations and the copyright notice.
% The command takes one argument, which is text to display at the start of the footnote.
% The \icmlEqualContribution command is standard text for equal contribution.
% Remove it (just {}) if you do not need this facility.

\printAffiliationsAndNotice{}  % leave blank if no need to mention equal contribution
% \printAffiliationsAndNotice{\icmlEqualContribution} % otherwise use the standard text.

\begin{abstract}
Score function estimation is the cornerstone of both training and sampling from diffusion generative models. Despite this fact, the most commonly used estimators are either biased neural network approximations or high variance Monte Carlo estimators based on the conditional score. We introduce a novel nearest neighbour score function estimator which utilizes multiple samples from the training set to dramatically decrease estimator variance. We leverage our low variance estimator in two compelling applications. Training consistency models with our estimator, we report a significant increase in both convergence speed and sample quality. In diffusion models, we show that our estimator can replace a learned network for probability-flow ODE integration, opening promising new avenues of future research.
%This document provides a basic paper template and submission guidelines.

\end{abstract}

\section{Introduction}
    
% Paragraph 1: Intro Diffusion Generative Models
Diffusion models \cite{sohl2015deep, ho2020denoising, song2020score} have emerged as a powerful class of generative models. They have seen broad adoption across a variety of tasks such as image generation \cite{rombach2022high}, video generation \cite{harvey2022flexible}, and 3D object synthesis \cite{poole2022dreamfusion}. Although these models achieve state of the art performance, their sampling procedure requires integration of the probability flow ODE (PF-ODE) or diffusion SDE \cite{song2020score}. This procedure hampers sampling speed as each integration step requires a neural network evaluation. Typically, many integration steps are required per sample \cite{ho2020denoising, karras2022elucidating}.

Motivated by this shortcoming, an array of methods have been proposed which learn few-step approximations of the PF-ODE solution. Amongst these, some attempt to distill a pre-trained diffusion network such as progressive distillation \cite{salimans2022progressive}, consistency distillation \cite{song2023consistency} or adversarial diffusion distillation \cite{sauer2023adversarial}. Alternatively, consistency training \cite{song2023consistency,song2023improved} proposes a method to train consistency models without a pre-trained diffusion model. We refer to the general class of models including diffusion models and related few-step methods as diffusion generative models.

% Paragraph 2: Introduce the score function as the critical quantity to estimate. 

As diffusion generative models are grounded in diffusion processes, they utilize the \emph{score function}. This critical quantity can be estimated by network approximation or Monte Carlo methods. Diffusion and consistency training utilize a one-sample Monte Carlo estimator of the score  in their training procedure, (excepting \cite{xu2023stable}). Alternatively, both diffusion sampling and diffusion distillation methods rely upon neural network estimators. However, each estimator has drawbacks. The single-sample Monte Carlo estimator has high variance \cite{xu2023stable} while learned neural network approximators are imperfect, leading to bias \cite{karras2022elucidating}.

% The basis of these few-step modelling approaches and the diffusion models which inspired them is the \emph{score function}. %which is required for integration of the probability flow ODE.
% Generally, this critical quantity is estimated by one-sample Monte Carlo Estimate or by a trained neural network approximator.
% one-sample Monte Carlo score estimates are used to train diffusion and consistency training, while a trained score model 
% neural network approximation or Monte Carlo estimationTo train a network score estimator, diffusion models and consistency training rely on a one-sample Monte Carlo estimate of the score function. during diffusion sampling and distillation training, a neural network estimator of the score function is employed. Each method has drawbacks. The Monte Carlo estimator has high variance \cite{xu2023stable} while the neural network estimator is imperfect, resulting in bias \cite{karras2022elucidating}.

% Paragraph 3: Our contriubtion
Our work puts forward a new method for score function estimation with lower variance than the one-sample estimator and less bias than neural network approximation. Utilizing self-normalized importance sampling \cite{hesterberg1995weighted}, we estimate the score through a weighted average over a batch of training set examples. We draw these elements from a proposal which prioritizes examples which are the most likely to have generated the current noisy value (\cref{fig:posterior_illustration}). We show that because diffusion processes are Gaussian, these examples are equivalent to the $\ell_2$ nearest neighbours to the noisy data. With this finding, we can rapidly identify and sample important examples using a fast KNN search \cite{johnson2019billion}. We analyze our method and derive bounds on the variance of our estimator.

% Paragraph 4: Our results
Empirically, we measure our performance in three settings. First, we compare our score estimate against the analytic score on CIFAR-10 \cite{krizhevsky2009learning}. We find that our method has near-zero variance and bias -- substantially outperforming both STF \cite{xu2023stable} and EDM \cite{karras2022elucidating}. Applying our method to consistency models, we find that replacing one-sample estimators with our method improves consistency training -- resulting in faster convergence and higher sample quality. Finally, we show that our method can replace neural networks in PF-ODE integration, opening interesting future research avenues. We release our code\footnote{\href{https://github.com/plai-group/knn-score}{https://github.com/plai-group/knn-score}} to encourage use of our estimator by the community.
\section{Background}
    Diffusion processes form the foundation of both diffusion models \cite{sohl2015deep, ho2020denoising} and consistency models \cite{song2023consistency}. In the variance exploding formulation \cite{song2020score}, diffusion processes convolve a data distribution $p(\bx)$, $\bx \in \mathbb{R}^d$ with zero-mean Gaussian noise $\mathcal{N}\left(\mathbf{0}, \sigma(t)^2\mathbf{I}\right)$, resulting in a continuous  marginal distribution path $p_t(\bz) = \int p_{\text{data}}\left(\bx\right)\mathcal{N}\left(\bz ; \bx, \sigma(t)^2\mathbf{I}\right)d\bx $, $t \in \left[t_{\mathrm{min}}, T\right], \bz \in \mathbb{R}^d.$ The noise schedule of the diffusion process $\sigma(t)$ and diffusion length $T$ are selected to ensure $p_T(\bz)$ is indistinguishable from a normal prior $\pi(\bz) = \mathcal{N}\left(\bz ; \mathbf{0}, \sigma(T)^2\mathbf{I}\right).$

Since the density $p(\bx)$ is usually unknown, it is approximated by $p_{\text{data}}(\bx)$, a finite mixture of Dirac measures $\pdata = \frac{1}{N}\sum_{\bxi \in \dataset} \delta \left(\bx - \bxi\right)$ where $\mathcal{D} = \left\{\bx^{(1)}, \ldots, \bx^{(N)}\right\}$ is a dataset of samples drawn from the true data density. %With this approximation, the marginals $p_t(\bz)$ can be described using a finite sum because the support of $\pdata$ is finite. T
With this approximation, the time-varying marginals $p_t(\bz)$ and posteriors $p_t(\bxi | \bz)$ take the forms 
\begin{align}
    &p_t(\bz) = \sum_{i=1}^N p_t(\bz | \bxi) p_{\text{data}}(\bxi) =  \sum_{i=1}^N \frac{p_t(\bz | \bxi)}{N}\label{eq:marginal}\\
    &p_t(\bxi | \bz) = \frac{p_t(\bz | \bxi)p_{\text{data}}\left(\bxi\right)}{p_t(\bz)} = \frac{p_t(\bz | \bxi)}{N\cdot p_t(\bz)} \label{eq:posterior}
\end{align} 
Under the finite data approximation, each posterior distribution $p_t(\bxi | \bz)$ is categorical over the elements of $\mathcal{D}$, ie. $\sum_{i=1}^N p(\bxi | \bz) = 1$. In addition, we refer to $p_t(\bz | \bxi) = \mathcal{N}\left(\bz ; \bx, \sigma(t)^2\mathbf{I}\right)$ as the forward likelihood.
%We also highlight the notation $p_t(\bxi | \bz)$ which refers to the categorical posterior distribution over $\mathcal{D}$ and not the continuous posterior $p_t(\bx| \bz)$ over the support of the data generating density $\pdata$.

Diffusion processes can be described using stochastic or ordinary differential equations \cite{song2020score}. In this work, we focus on the probability flow ODE (PF-ODE) formulation, introduced by \citet{song2020score}. Following the parameterization of EDM \cite{karras2022elucidating}, we select $\sigma(t) = t$. However, other choices are possible, as discussed in \cref{ap:general_diffusion}. With our chosen diffusion schedule, the PF-ODE is given by
\begin{equation}
    d\bz = -t \ \nabla_\bz \log p_t(\bz) dt. \label{eq:pfode}
\end{equation}
Numerical integration of the PF-ODE requires repeated evaluation of $\nabla_\bz \log p_t(\bz)$, known as the \emph{score function}. Because the forward likelihood is Gaussian, the \emph{conditional} score has the straightforward form $\nabla_\bz \log p_t(\bz | \bxi) = \frac{\bxi - \bz}{t^2}$. However, calculating the \emph{marginal} score for fixed $\bz$ and $t$ requires an expectation over the posterior $p_t(\bxi | \bz)$
\begin{align}
    \nabla_\bz \log p_t(\bz) &= \mathop{\mathbb{E}}_{\bxi \sim p_t(\bxi | \bz)}\left[\nabla_\bz \log p_t(\bz | \bxi)\right] \label{eq:score_identity}\\
    &= \frac{\mathbb{E}[\bx | \bz, t] - \bz}{t^2} \label{eq:clean_score}
\end{align}
where $\mathbb{E}\left[\bx | \bz, t\right]$ is the mean of $p_t(\bxi | \bz)$. 

Since calculating posterior probabilities via \cref{eq:posterior} involves summing over $\mathcal{D}$ to find $p_t(\bz)$, exact evaluation of \cref{eq:clean_score} is computationally prohibitive outside of  small data regimes. Instead, most methods which require the score function must use estimators of various kinds. In the following sections, we highlight the use of score estimators in diffusion and consistency models. 

\subsection{Diffusion Models}

Diffusion models are a class of generative models which can generate samples through numerical integration of \cref{eq:pfode} from $T$ to $t_{\mathrm{min}} \approx 0$ with initial value $\bz \sim \pi(\mathbf{\bz})$. 

In place of the exact score function discussed previously, diffusion models use learned neural network estimators of the marginal score $\mathbf{s}_\theta(\mathbf{z}, t)$ to generate samples. Following \cref{eq:clean_score}, $\mathbf{s}_\theta$ can be parameterized as $\mathbf{s}_\theta(\bz, t) = \frac{D_\theta(\bz, t) - \bz}{t^2}$ where $D_\theta$ is a learned estimator of the posterior mean $\mathbb{E}[\bx | \bz, t]$. Training is performed via stochastic gradient descent on a denoising score matching objective $\mathcal{L}_{DSM} = \mathbb{E}_t \left[ \lambda(t) \mathcal{L}_t \right]$ with weights $\lambda(t)$. $\mathcal{L}_{t}$ is given by 
\begin{equation}
    \mathcal{L}_t = \mathbb{E}_{\bz, \bxi \sim p_t(\bz, \bxi)}\left[\lVert \mathbf{D}_\theta\left(\mathbf{z}, t\right) - \bxi \rVert_2^2 \right]. \label{eq:DSM}
\end{equation}
The minimizer of \cref{eq:DSM} is $D^\star_\theta\left(\bz, t\right) = \pmean$ \cite{karras2022elucidating}. However, during stochastic gradient descent optimization, per-batch loss is calculated using $(\bxi, \bz)$ tuples drawn from the joint distribution instead of against the true posterior mean $\mathbb{E}[\bx | \bz, t]$. Minimizing the loss with individual $\bxi$  converges to $\pmean$ because the joint distribution from which the tuples are drawn $p_t(\bz, \bxi)$ is proportional to $ p_t(\bxi | \bz)$. Due to this relation, each $\bxi$ can be seen as a single-sample Monte Carlo estimator of the posterior mean \cite{xu2023stable}.
%%%%%%%%%%%%%%%%%%%%%%%%%%%%%%%%%%%%%%%%%%%%%%%%%%%%%%%%%%%%%%%%%%%%%%%%%%%%%%%5
\begin{figure*}[ht!]
    \centering
    \includegraphics[width=\textwidth]{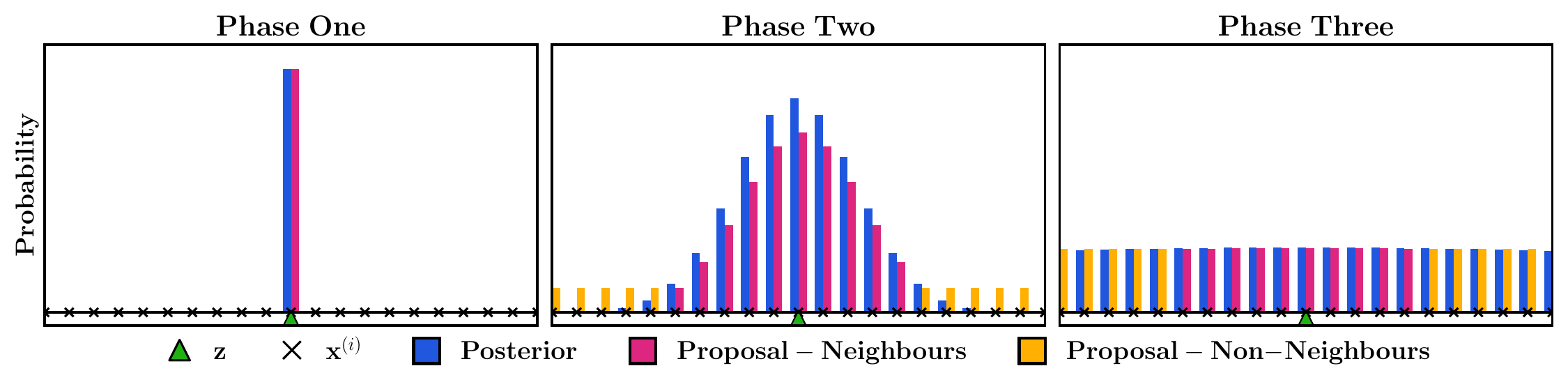}
    \caption{Illustration of our proposal and the posterior across three phases of a toy 1D diffusion process. \textbf{Left:} For small $t$, the posterior probability is concentrated on the single, closest element to $\bz$. \textbf{Middle:} For intermediate $t$, we upper bound the posterior probability for non-neighbour elements, resulting in under weighting neighbours. \textbf{Right:} As $t$ becomes large, the posterior approaches a uniform distribution and the proposal matches the posterior well.}
    \label{fig:posterior_illustration}
\end{figure*}

\subsection{Consistency Models}
Sampling from diffusion models requires repeated evaluation of the score estimator $s_\theta(\bz, t)$. To address this shortcoming, consistency models \cite{song2023consistency} learn a one step mapping between $\pi(\bz)$ and the data distribution.

\Cref{eq:pfode} defines trajectories which map between any marginal distribution $p_t(\bz)$ and  $p_{t_{\mathrm{min}}}(\bz) \approx \pdata$, where $t_{\mathrm{min}} \approx 0$ is used for numerical stability. These mappings $\bm{f} : (\bz, t) \rightarrow \bx$ are known as \emph{consistency functions} because the value $\bm{f}(\bz,t) = \bx^\star$  is consistent for all $(\bz,t)$ on a common PF-ODE trajectory. Consistency models \cite{song2023consistency} exploit this self-consistency property to learn $\bm{f}_{\bm{\theta}}(\bz, t)$, a neural network parameterized approximation to the true consistency function $\bm{f}$.
In practice, consistency models are trained by optimizing the consistency matching objective $\mathcal{L}_{CM} = \mathbb{E}_{t_i} \left[ \lambda(t_i) \mathcal{L}_{t_i} \right]$ where
\begin{equation}
    \mathcal{L}_{t_i} = \mathbb{E}\left[\bm{d}\left(\student\left(\bz, t_i\right), \teacher\left(\hat{\bz}, t_{i-1}\right)\right) \right]. \label{eq:CM}
\end{equation}
The consistency matching objective is optimized with respect to $\theta$ over $N$ discrete time steps $t_{\mathrm{min}} = t_1 < \ldots < t_N = T$, with an expectation over $\bxi \sim \pdata$, $i \sim p(i)$ and $\bz \sim p_{t_i}(\bz | \bx)$. Here $p(i)$ is a categorical distribution over integers $i \in [2, N]$, $\lambda(t_i)$ is a weighting function, and $d(\bx, \mathbf{y})$ is a metric function such as $\ell_2$, LPIPS \cite{zhang2018unreasonable}, or Pseudo-Huber  \cite{charbonnier1997deterministic}.

The loss aims to ensure consistency between the student network $\student$ and the teacher network $\teacher$, evaluated at points $(\bz, t_i)$ and $(\hat{\bz}, t_{i-1})$ which lie on the same PF-ODE trajectory. To obtain $\hat{\bz}$, consistency model training performs a one step integration of \cref{eq:pfode} from from $t_i$ to $t_{i-1}$, setting $\hat{\bz} = \texttt{solver} \left( \bz, t_i, t_{i-1}, \nabla_{\bz} \log p_{t_i}( \bz ) \right)$. An accurate score function estimate is critical to performing this integration, ensuring $\bz$ and $\hat{\bz}$ lie on the same PF-ODE path.

Since $\nabla_{\bz} \log p_{t_i}( \bz )$ is infeasible to calculate exactly outside the small data regime, consistency models are trained with score estimators. \citet{song2023consistency} propose training consistency models with two such estimators. Consistency distillation utilizes a neural network approximator $\nabla_{\bz} \log p_{t_i}( \bz ) \approx s_\theta(\bz, t_i)$ from a pre-trained diffusion ``teacher'' model. To train consistency models without a pre-trained score estimator, consistency training utilizes the estimator $\nabla_{\bz} \log p_{t_i}( \bz ) \approx \frac{\bxi - \bz}{t^2}$ with $\bxi$ and $\bz$ sampled from $p_t(\bz, \bxi).$ Like diffusion training, a single-sample Monte Carlo estimate is used in place of the marginal score.  
\section{Nearest Neighbour Score Estimators} % Two Method Sections?
    \subsection{SNIS Score Estimation}
Although the score is necessary for both training and sampling from diffusion generative models, there are only two widely used estimators of the score function -- learned networks or single-sample Monte Carlo estimators. Since the latter is required to produce the former, we consider methods to reduce variance in Monte Carlo score estimation.

When calculating the score via \cref{eq:clean_score}, the unknown quantity is $\pmean$ which we will refer to  as $\mu$ hereafter
\begin{equation}
    \mu = \pmean = \sum_{i=1}^N p_t(\bxi | \bz) \ \bxi. \label{eq:posterior_mean}
\end{equation}
If $p_t(\bxi | \bz)$ is available, Monte Carlo estimation can be used to convert the sum over $\mathcal{D}$ in \cref{eq:posterior_mean} into a sum over a batch $S = \left\{\bx_1, \ldots, \bx_n \right\} \sim p_t(\bxi | \bz)$ of size $n$
\begin{equation}
    \hat{\mu}_{\text{MC}} = \frac{1}{n}\sum_{\bxi \in S} \bxi.\label{eq:simple_mc}
\end{equation}
The previously discussed single-sample Monte Carlo estimator for diffusion and consistency training is an example of \cref{eq:simple_mc} with $n=1$ and $\bx_1 \sim p_t(\bxi, \bz) \propto p_t(\bxi | \bz)$. This estimator is unbiased but the variance per-dimension scales with $1/n$ \cite{mcbook}. Unfortunately, drawing $n>1$ samples to reduce estimator variance is challenging because $p_t(\bxi | \bz)$ is expensive to compute. The normalizing constant $p_t(\bz)$ can only be found by summing over $\mathcal{D}$, so naive multi-sample Monte Carlo estimation is impractical.

% Despite this, both diffusion and consistency training utilize a batch of size $n=1$. To sample the batch, first $\bxi \sim \pdata$ and then $\bz \sim p_t(\bz | \bxi)$. Since the $p_t(\bxi, \bz) \propto p_t(\bxi | \bz)$, $\bxi$ can be considered a posterior sample conditioned on $\bz$. For a fixed $\bz$ determining the categorical posterior distribution via \cref{eq:posterior} requires summing over $\mathcal{D}$, making drawing additional samples computationally infeasible in most circumstances.

To increase $n$ and thereby decrease estimator variance, we appeal to importance sampling (IS)\cite{kahn1953methods} taking inspiration from \citet{xu2023stable}. Instead of sampling from $p_t(\bxi | \bz)$, we draw samples from a proposal distribution $q(\bxi)$. We can formulate a new Monte Carlo estimator using a batch of $n$ samples drawn from the proposal $S_q = \left\{\bx_1, \ldots, \bx_n\right\} \sim q(\bxi)$ as 
\begin{equation}
    \hat{\mu}_{\text{IS}} = \frac{1}{n} \sum_{\bxi \in S_q} \frac{p_t\left(\bz | \bxi\right)p_{\text{data}}\left(\bxi\right)} {p_t\left(\bz\right)q(\bxi)} \bxi. \label{eq:IS_estimator}
\end{equation}
% 3d. Computing the importance ratio requires computing p(z) which involves summing over the entire dataset
% 4. Self-normalized importance sampling
    % We can compute p(z) with equation 1d.
    % Arrive at final self normalized importance sampler
\Cref{eq:IS_estimator} no longer requires samples from the posterior and with the finite data approximation, we know $p_{\text{data}}(\bxi)=1/N$. However, finding $p_t(\bz)$ still requires a sum over $\mathcal{D}$. We resolve this issue with self-normalized importance sampling \cite{hesterberg1995weighted}. By also estimating $p_t(\bz)$ with an IS estimator of \cref{eq:marginal} using the same $S_q$, we arrive at a self-normalized importance sampling (SNIS) estimator
\begin{equation}
     \hat{\mu}_{\text{SNIS}} = \frac{\sum_{\bxi \in S_q}w_i \bxi}{\sum_{\bx^{(j)} \in S_q} w_j}, \ w_i = \frac{p_t(\bz | \bxi)}{q(\bxi)}. \label{eq:SNIS_estimator}
\end{equation}
The constant factor $p_{\text{data}}(\bxi)$ cancels because it is present in the numerator of \cref{eq:IS_estimator} and in \cref{eq:marginal}.

Combining \cref{eq:SNIS_estimator} with \cref{eq:clean_score} yields a SNIS estimator of the marginal score. Although the SNIS estimator has some bias, by drawing $n > 1$ samples we can greatly reduce the variance of score estimation. 

\subsection{Nearest Neighbour Proposals}
With a defined a framework for reducing score estimator variance, we must next determine a suitable proposal distribution $q(\bxi)$. To identify a good proposal, we start by analyzing the variance of \cref{eq:SNIS_estimator} \cite{mcbook}. The diagonal of the SNIS score estimator covariance is
\begin{equation}
\begin{split}
     \text{Diag} & \left(\text{Cov}\left(\frac{\hat{\mu}_{\text{SNIS}} - \bz}{t^2} \right) \right)\\
    &= \frac{1}{nt^4}\sum_{i=1}^N \frac{p_t\left(\bxi | \bz \right)^2}{q\left(\bxi\right)} \left( \bxi - \mu \right)^{\circ 2}
\end{split} \label{eq:SNIS_Var}
\end{equation}
Where ${}^{\circ2}$ indicates the Hadamard power operator.

Examining \cref{eq:SNIS_Var}, we can infer that an appropriate proposal distribution $q(\bxi)$ is important for estimating $\pmean$ with low per-dimension variance. Specifically, high importance ratios $p_t(\bxi | \bz) / q(\bxi)$ will increase variance if $(\bxi - \mu)^{\circ 2}$ is non-zero.

For univariate SNIS estimators, the optimal proposal $q^\star(x)$ for a target distribution $p(x)$ with true mean $\mu = \mathbb{E}_p[x]$ is $q^\star\left(x\right) \propto p\left(x\right) \left|x - \mu \right|$ \cite{kahn1953methods}. However, $q^\star$ is not directly applicable to our problem because $\bxi$ is multivariate and the true mean $\mu=\pmean$ is unknown.

Although we cannot make immediate use of the univariate optimal proposal, we note that $q^\star$ is proportional to the target distribution $p(x)$. We therefore hypothesize that a suitable choice for our proposal is to match the true posterior $ p_t(\bxi | \bz)$ as closely as possible. Because $p_t(\bxi | \bz)$ varies with $t$ and $\bz$, we propose also varying our proposal with these variables, denoting our proposal as $q_t(\bxi | \bz)$.

Determining the posterior exactly requires summing over $\mathcal{D}$ to compute $p_t(\bz)$. To more efficiently construct our proposal, we structure $q_t(\bxi | \bz)$ to approximate $p_t(\bxi | \bz)$ by using only $k$ most probable elements of $p_t(\bxi | \bz)$. Let $\mathcal{K}(\bz,t) = (\bx_1, \ldots \bx_k | \bx_i \in \mathcal{D}) \ s.t. \ p_t(\bx_1 | \bz) \geq \ldots \geq p_t(\bx_k | \bz)$ be the ordered set of the $k$ most probable dataset elements under the posterior distribution $p_t(\bxi | \bz)$. The subscript of each $\bx_j \in \mathcal{K}(\bz,t)$ indicates its ordering, where $\bx_k$ is the  element of $\mathcal{K}(\mathbf{z}, t)$ with the lowest probability. For brevity, we refer to $\mathcal{K}(\bz, t)$ as $\mathcal{K}$ hereafter and use it to define our nearest neighbours proposal
\begin{equation}
    q_t(\bx^{(i)} | \ \bz) = \frac{1}{\mathcal{Z}_q} \begin{cases}
        p_t\left(\bz | \bxi\right) & \bxi \in \mathcal{K} \\
        p_t\left(\bz | \bx_k \right) & \bx^{(i)} \notin \mathcal{K}
    \end{cases}. \label{eq:proposal}
\end{equation}

\begin{algorithm}[t!]
   \caption{Nearest Neighbour Score Estimator}
   \label{alg:estimator}
\begin{algorithmic}
    \STATE {\bfseries Input:} $\bz$, $t$, $\mathcal{D}$,  $\mathcal{I}$,$k$,$n$
    % Get neighbours of z
    \STATE $\mathcal{K}, d \leftarrow \texttt{search}(\mathcal{I}, \bz, k)$
    \STATE $p_t(\bz | \bx_i) \leftarrow \text{exp}(\frac{-d_i^2}{2t^2}) \ \forall \ \bx_i \in \mathcal{K}$ 
    % \STATE $q_t(\bx^{(i)} | \bz) \leftarrow \text{exp}(\frac{-d_k^2}{2t^2}) \ \forall \ \bxi \notin \mathcal{K}$
    \STATE $\mathcal{Z}_q \leftarrow \sum_{\bx_i } p_t(\bz | \bx_i) + (N-k)p_t( \bz | \bx_k)$
    \STATE $q_t(\bx^{(i)} | \bz) \leftarrow \frac{1}{\mathcal{Z}_q } \begin{cases}
        &p_t (\bz | \bx_i)\ \forall \ \bxi \in \mathcal{K} \\
        & p_t (\bz | \bx_k)\ \forall \ \bxi \notin \mathcal{K}
    \end{cases} $
    \STATE Sample $\left\{\bx_1, \ldots, \bx_n \right\} \sim q_t(\bxi | \bz)$
    \STATE $w_i \leftarrow p(\bz | \bx_i) / q_t(\bx_i | \bz)$
    \STATE $\bar{w}_i \leftarrow \frac{w_i}{\sum_{j=1}^n w_j}$
    \STATE $\hat{\mu}_{\text{KNN}} \leftarrow \sum_{i=1}^n \bar{w}_i \bx_i$
    \STATE \textbf{return} $\frac{\hat{\mu}_{\text{KNN}} - \bz}{t^2}$  
    % Build proposal
    % Sample from proposal
    % Compute proposal weights
    % Normalize proposal weights
    % return Weighted average
\end{algorithmic}
\end{algorithm}
% \end{figure}

\begin{figure*}[ht!]
    \centering
    \begin{minipage}[b]{.5\textwidth}
        \centering
        \includegraphics[width=\textwidth]{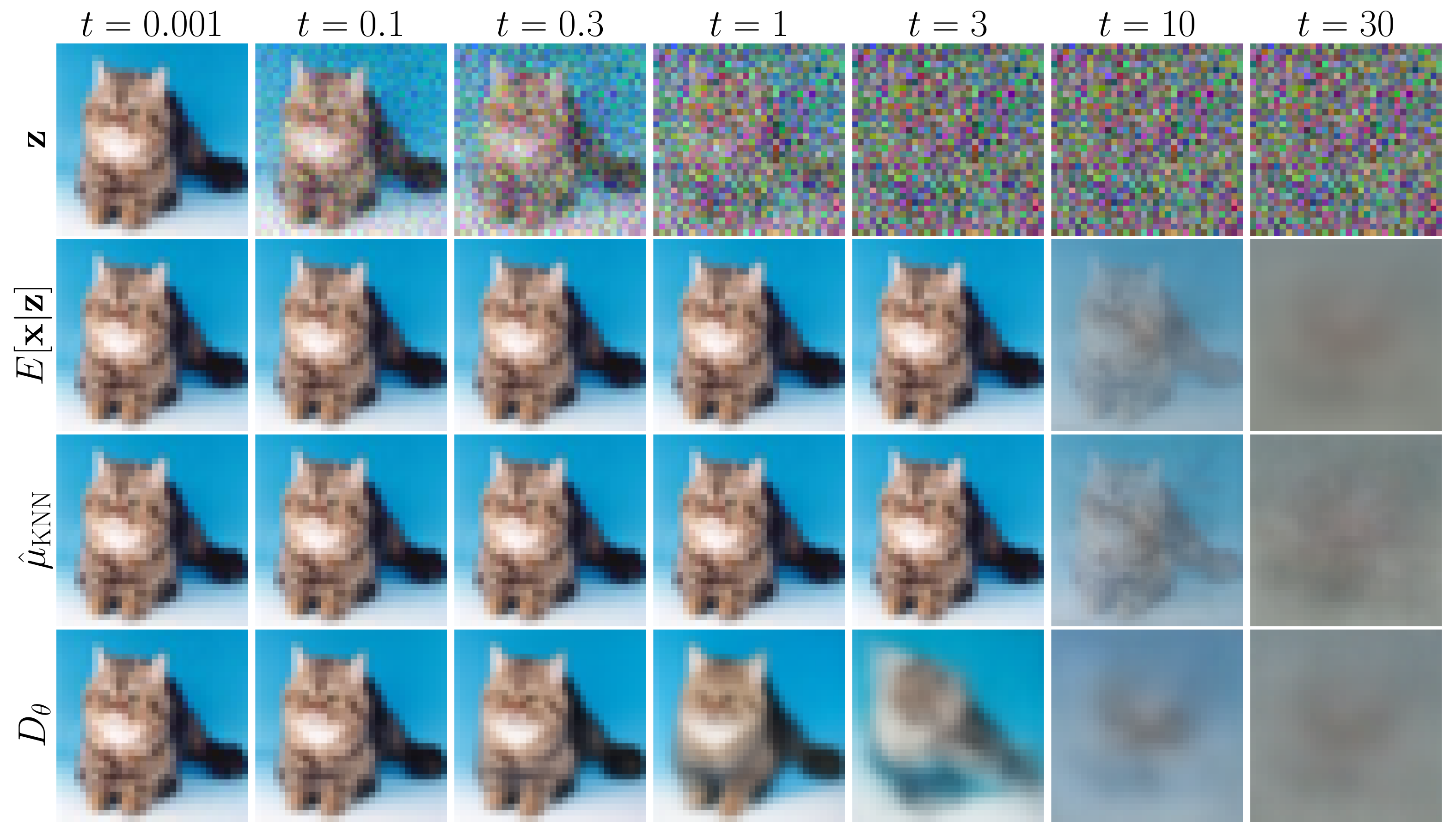}
    \end{minipage}\hfill
    \begin{minipage}[b]{.5\textwidth}
        \centering
        \includegraphics[width=\textwidth]{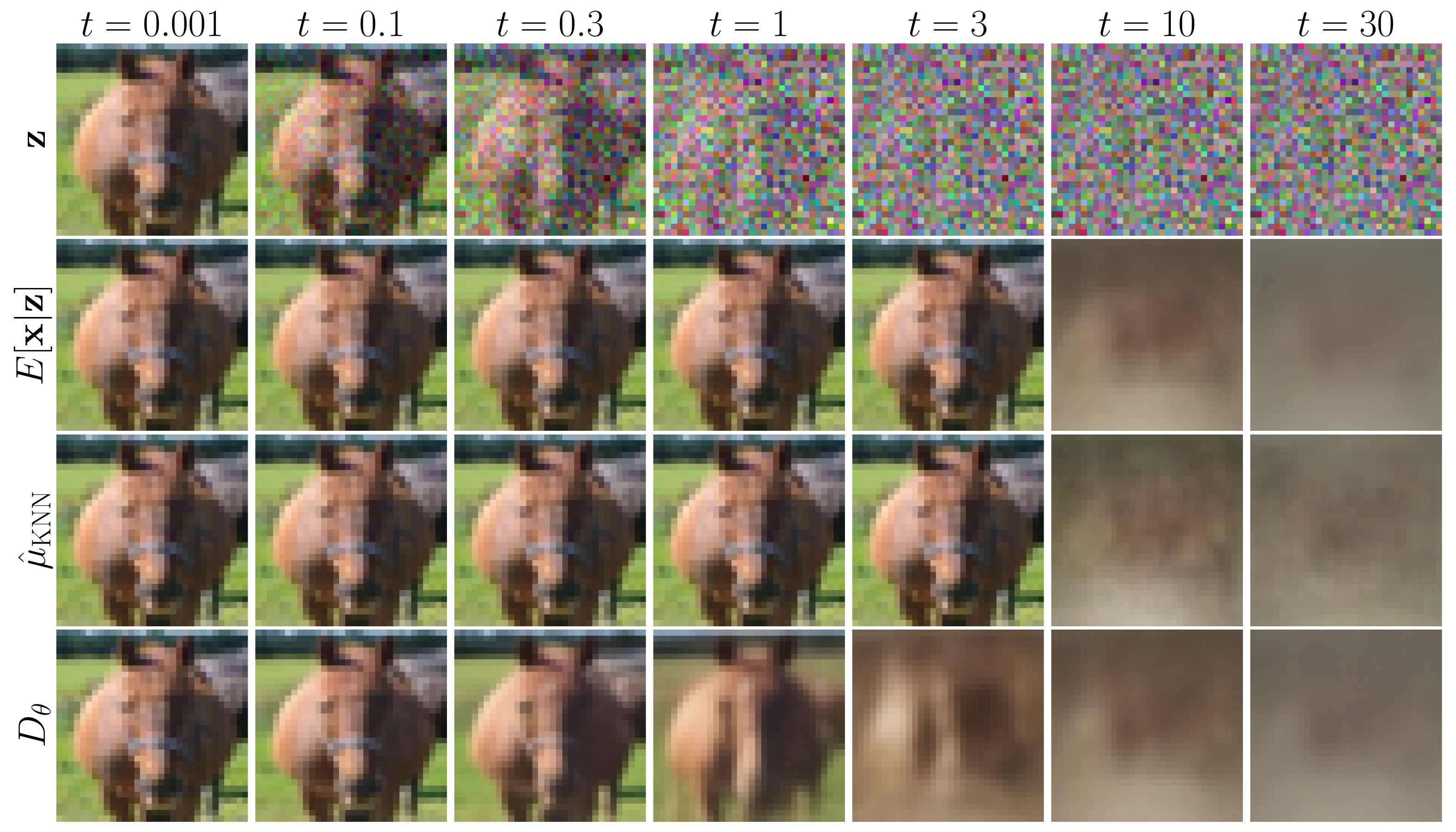}
    \end{minipage}
    \caption{Visualization of posterior mean estimators on CIFAR-10 images. \textbf{Top:} Noisy $\bz$ samples from $p_t(\bz | \bxi)$ for increasing noise levels. \textbf{Second Row:} True posterior mean. \textbf{Third Row:} Posterior mean estimates from our estimator. Our estimator nearly perfectly matches the true mean levels. \textbf{Bottom:} Estimated posterior mean from a trained diffusion model. The diffusion model does not match the high frequency features for low noise levels, but has fewer artifacts at the highest noise level than our method.}
    \label{fig:denoiser_viz}
\end{figure*}

There are several beneficial properties of our nearest neighbours proposal. Most importantly, \cref{eq:proposal}, approximates the posterior using only the elements of $\mathcal{K}$, instead requiring a full summation over $\mathcal{D}$. In addition, since $p_t(\bz | \bx) = \mathcal{N}(\bz | \bx, t^2\mathbf{I})$ is Gaussian, and the size of the $\mathcal{D}$ is known, the normalizing constant $\mathcal{Z}_q = \sum_{i=1}^k p_t(\bz | \bx_i) + (N-k)p_t(\bz | \bx_k)$ is found easily. Finally, the shape of $q_t(\bxi | \bz)$ matches the shape of $p_t(\bxi | \bz)$, as demonstrated in \cref{fig:posterior_illustration}.

% \begin{figure}[h!]

Examining \cref{fig:posterior_illustration}, for $\bxi \in \mathcal{K}$, the proposal is proportional to the posterior with $\mathcal{Z}_q q_t(\bxi | \bz) = p_t(\bz)p_t(\bxi | \bz)$. In general, $p_t(\bz) \leq \mathcal{Z}_q$ with equality in two cases. When $p_t(\bz | \bx_k) = 0$ or $k = N$, the posterior is fully concentrated on $\mathcal{K}$ and $p_t(\bz) = \sum_{i=1}^N p_t(\bz | \bxi) = \sum_{i=1}^k p_t(\bz | \bx_i)$. In these cases, our nearest neighbour proposal exactly matches the posterior distribution.

When $p_t(\bz | \bx_k) > 0$ and $k < N$, the likelihood $p_t(\bz | \bx_k)$ upper bounds the likelihood for all $\bxi \notin \mathcal{K}$. In this case, we underweight the probability for $\bxi \in \mathcal{K}$ and assign more mass to the tails of the proposal distribution, as seen in \cref{fig:posterior_illustration}b.

\Cref{eq:proposal} requires access to the $k$ most probable elements of the posterior. To quickly identify $\mathcal{K}$, we rely on the Gaussian nature of diffusion processes. Since $p_t(\bxi | \bz) \propto p_t(\bz | \bxi)$ and $p_t(\bz | \bxi)$ is Gaussian, the $k$ most probable elements of the posterior are equivalent to the $k$ elements of $\mathcal{D}$ with the smallest $\ell_2$ distance to $\bz$. Therefore, we transform the difficult problem of finding the $k$ most likely elements of the posterior into the easier problem of identifying the $k$ nearest neighbours of $\bz$ in Euclidean space. 

To determine $\mathcal{K}$ in practice, we utilize FAISS \cite{johnson2019billion} to perform fast $k$ nearest neighbour search over $\mathcal{D}$. To construct $q_t(\bxi | \bz)$ for a given $\bz$, we first use FAISS to find the nearest neighbours of $\bz$. Helpfully, FAISS returns both the nearest neighbour set $\mathcal{K}$ and their $\ell_2$ distances $d \in \mathbb{R}^k$ to $\bz$. Using $d$, we compute the forward likelihood for each element of $\mathcal{K}$, and use \cref{eq:proposal} to compute the proposal. We can then use the proposal to estimate the score via \cref{eq:SNIS_estimator}. \Cref{alg:estimator} outlines our full nearest neighbour score estimation algorithm.
\section{Estimator Performance}
    \begin{figure*}[ht!]
    \centering
    \begin{minipage}[b]{0.7317073170731707\textwidth}
        \centering
        \includegraphics[width=\textwidth]{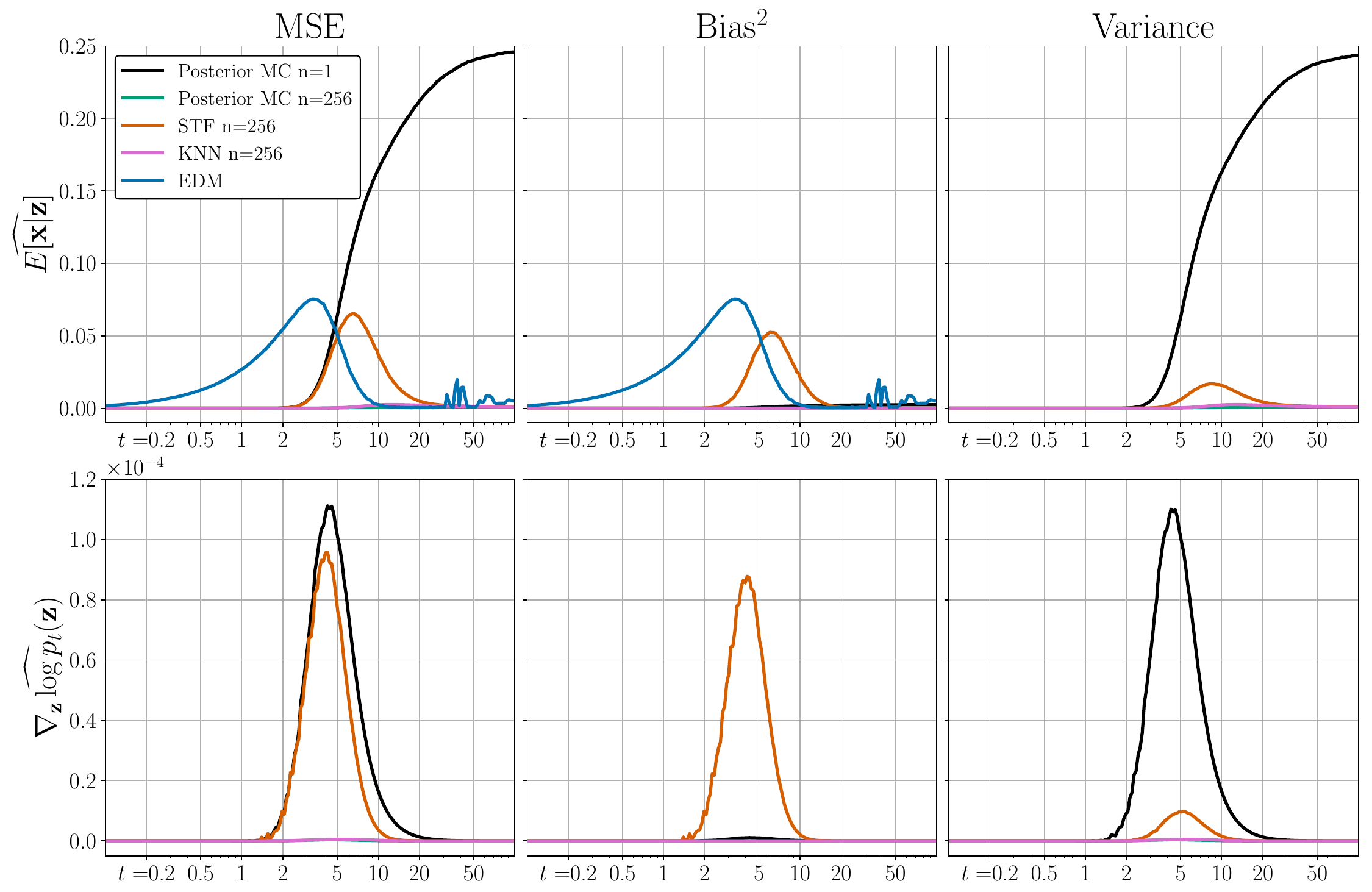}
    \end{minipage}\hfill
    \begin{minipage}[b]{0.2682926829268293\textwidth}
        \centering
        \includegraphics[width=\textwidth]{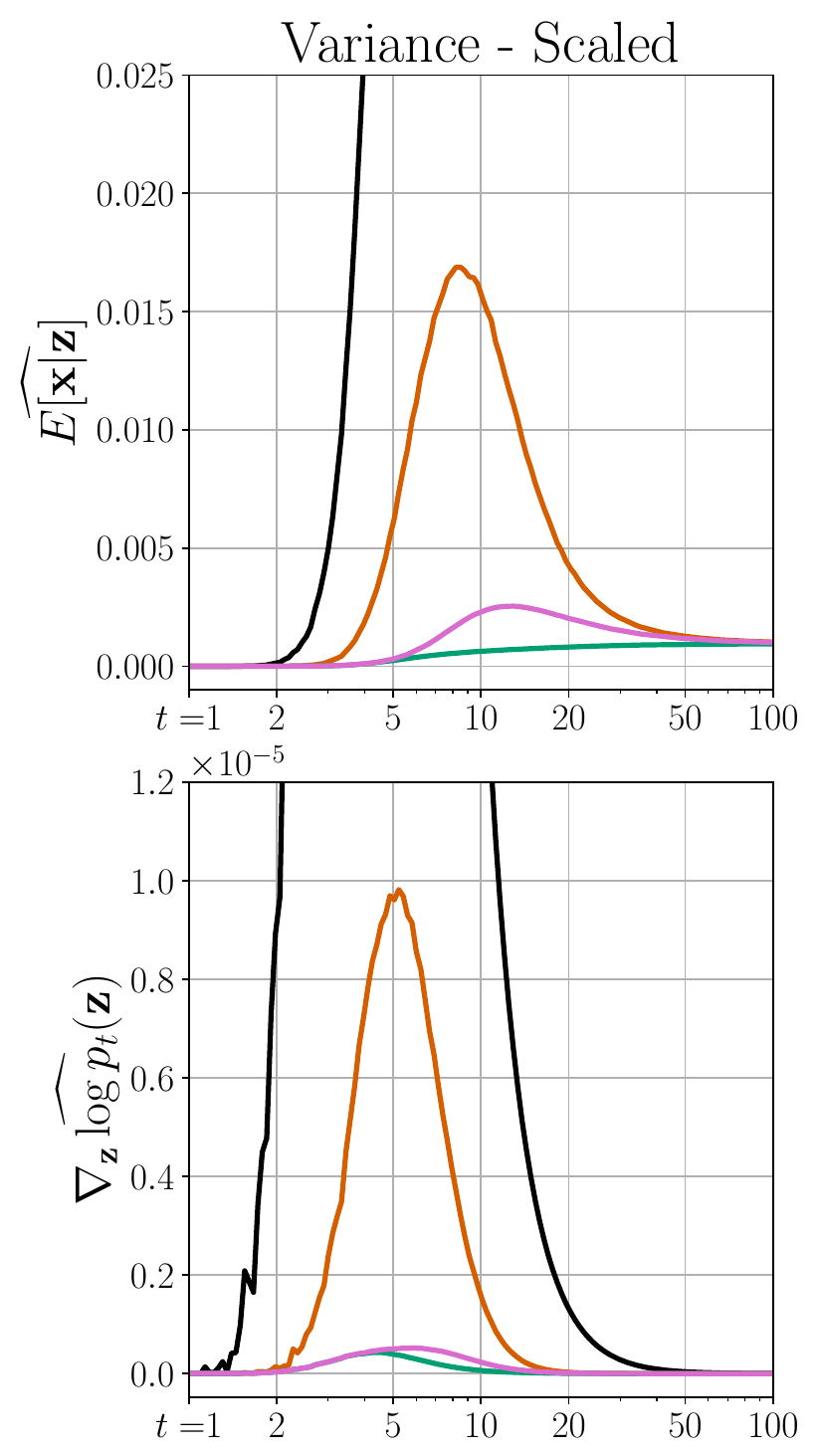}
    \end{minipage}
    \caption{Estimator performance on CIFAR-10. Our estimator reduces bias and variance to near zero, significantly outperforming even a network score estimator. In contrast, STF reduces variance but has significant bias for intermediate $t$.}
    \label{fig:estimator_perf}
\end{figure*}
\subsection{Analysis} \label{sec:analysis}

Motivated by our goal of reducing the variance of the score estimate, we derive two theoretical bounds on the trace of the covariance of $\hat{\mu}_{\text{KNN}}$ -- our posterior mean estimator. We form our bounds using the trace of the covariance of a Monte Carlo estimator $\hat{\mu}_{MC}$ and another SNIS estimator with $q_t(\bxi | \bz) = \frac{1}{N}$ which we denote with $\hat{\mu}_{U}$. We note that $\hat{\mu}_{U}$ is is only asymptotically equivalent to STF \cite{xu2023stable} because in addition to a uniform $q_t(\bxi | \bz)$, STF also deterministically includes the source $\mathbf{x}^{(i)}$ in their sample batch $S$.

\begin{theorem}
\label{thm:knn_bound}
Let $t \in (0, \infty)$, $\hat{\mu}_{\text{MC}}$ be the Monte Carlo estimator defined by \cref{eq:simple_mc}, and $\hat{\mu}_{\text{KNN}}$ be the estimator described by \cref{eq:SNIS_estimator} with proposal given by \cref{eq:proposal}. Then for a fixed $n \geq 1$,
\begin{equation}
    \text{Tr}\left(\text{Cov}\left(\hat{\mu}_{\text{KNN}}\right)\right)\leq \frac{\mathcal{Z}_q}{p_t(\bz)}\text{Tr}\left(\text{Cov}\left(\hat{\mu}_{\text{MC}}\right)\right).
\end{equation}
\end{theorem}
\begin{proof}
We defer the proof to \cref{ap:knn_upper}.
\end{proof}
In general, our estimator's trace-of-covariance is larger than the simple Monte Carlo estimator's trace-of-covariance. The ratio of the two is upper bounded by the factor $\frac{\mathcal{Z}_q}{p_t(\bz)}$ -- the ratio of the normalizing constants of the proposal and the posterior distributions. This ratio will be exactly unity when $p_t(\bxi | \bz)$ is fully concentrated on $\mathcal{K}$ (as is true when $t$ is small or $k=N$). The ratio will also approach unity as $p_t(\bxi | \bz)$ approaches a uniform distribution over $\mathcal{D}$. In any of these cases, the trace-of-covariance of our estimator will approach the trace-of-covariance of multi-sample posterior Monte Carlo.
\begin{theorem}
\label{thm:knn_vs_stf}
Let $\hat{\mu}_{\text{MC}}$ be defined by \cref{eq:simple_mc}. Let $\hat{\mu}_{\text{KNN}}$ and $\hat{\mu}_{\text{U}}$ be two estimators defined by \cref{eq:SNIS_estimator} with proposals given by \cref{eq:proposal} and $q_t(\bxi | \bz) = \frac{1}{N}$ respectively. Then, for $t \in (0, \infty)$ and a fixed $n \geq 1$,
\begin{align*}
     \text{Tr}\left(\text{Cov}\left(\hat{\mu}_{\text{KNN}}\right) \right) &\leq \\
    \bigg( & 1 - \frac{\sum_{\bxi \notin \mathcal{K}} p_t(\bz | \bxi)}{p_t(\bz)} \bigg) \text{Tr}\left(\text{Cov}\left(\hat{\mu}_{\text{MC}}\right)\right) \\
    + &\frac{(N-k)}{N} \text{Tr}\left(\text{Cov}\left(\hat{\mu}_{\text{U}}\right)\right).
\end{align*}
\end{theorem}
\begin{proof}
We defer the proof to \cref{ap:knn_v_stf}.
\end{proof}
Our estimator's trace-of-covariance is upper bounded by the sum of a fixed multiple of  the uniform SNIS estimator's trace-of-covariance and a time-varying multiple of the posterior Monte Carlo trace-of-covariance. We note that when $p_t(\bxi | \bz) = \frac{1}{N} \ \forall \ \bxi \in \mathcal{D}$, then $\text{Tr}\left(\text{Cov}\left(\hat{\mu}_{\text{MC}}\right)\right) = \text{Tr}\left(\text{Cov}\left(\hat{\mu}_{\text{U}}\right)\right)$, the frst term coefficient becomes $\frac{k}{N}$, and the trace-of-covariance of $\hat{\mu}_{\text{KNN}}$ and $\hat{\mu}_{\text{U}}$ are equal. For intermediate $t$ we expect the trace-of-covariance of $\hat{\mu}_{\text{KNN}}$ to be lower than that of $\hat{\mu}_{\text{U}}$ because we expect $\text{Tr}(\text{Cov}(\hat{\mu}_{\text{MC}})) << \text{Tr}(\text{Cov}(\hat{\mu}_{\text{U}}))$ in this region.

% Introduce bound on asymptotic variance

% Introduce the metric we use to measure the estimator performance
\subsection{Empirical Performance} \label{sec:empirical_perf}

To gain insight into the performance of our estimator beyond our theoretical analysis, we empirically compare our estimator against various other score estimators on the CIFAR-10 dataset \cite{krizhevsky2009learning}. We measure estimator performance with average mean squared error
\begin{align}
        \text{MSE} & \left(\hat{\mu}\right)= \frac{1}{d}\mathop{\mathbb{E}}\left\lVert \hat{\mu} - \mu \right\rVert_2^2 = \frac{1}{d}\left(\text{Bias}\left(\hat{\mu}\right)^2 + \text{Var}\left(\hat{\mu}\right)\right). \label{eq:mse}
\end{align}
\Cref{eq:mse} compares the estimated posterior mean $\hat{\mu}$ against the true posterior mean $\mu$,  calculated via \cref{eq:posterior_mean}. MSE is equivalent to the trace-of-covariance metric introduced in \cite{xu2023stable} and discussed in \cref{sec:analysis} scaled by $1/d$. The outer expectation of \cref{eq:mse} is calculated over a large set of $\bxi \sim \mathcal{D}$, with one $\bz \sim p_t(\bz | \bxi)$ for each $\bxi$. For each $\bz$, we average over repeated estimator evaluations. This allows us to decompose MSE into squared bias and variance terms to further investigate estimator performance. When calculating variance and squared bias, we compute the sample mean from the repeated estimator evaluations per-$\bz$. 

In \cref{fig:estimator_perf}, we plot score and posterior mean estimator performance across a range of noise levels. At each $t$, we evaluate the estimators using ten thousand $\bz$ samples, and one hundred estimator evaluations per $\bz$. We compare our method against four others: the single-sample posterior estimator typically used in diffusion and consistency training, a $n=256$ multi-sample posterior Monte Carlo estimator, an STF \cite{xu2023stable} estimator with $n=256$, and EDM -- a pre-trained near-SoTA diffusion model \cite{karras2022elucidating}. As EDM is deterministic, we do not report its variance. Further, we do not plot its score metrics since errors in the posterior mean estimation for small $t$ are amplified in the score by a factor of $1/t^4$. For the multi-sample Monte Carlo estimator, we calculate the posterior by calculating \cref{eq:marginal}.

Examining \cref{fig:estimator_perf}, we note several important results. Firstly, our estimator outperforms the STF estimator in both posterior mean and score estimation across every metric. The difference is especially stark for score estimation, where the peak MSE of our estimator is approximately 100 times better than that of STF. We hypothesize that our excellent performance is because for an intermediate range of $t$, $q_t(\bxi | \bz)$ matches the posterior nearly perfectly. When $p_t(\bxi | \bz)$ is fully concentrated on a single element, the variance of the SNIS estimators is zero because $p_t(\bxi | \bz)=1 \implies (\bxi - \pmean)^2=0 $ and  $(\bxi - \pmean)^2 > 0  \implies p_t(\bxi | \bz)=0$. However, as $t$ increases, the number of elements of $\mathcal{D}$ with non-zero $p_t(\bxi | \bz)$ also increases, resulting in a non-zero estimator variance. In these cases, we believe the variance is dominated by the importance ratio $p_t / q_t$. Compared to STF, our proposal has lower importance ratios when the posterior mass is concentrated on less than $k$ elements.

Our second finding is that the STF estimator has a significant bias for intermediate $t$. We believe this bias is because STF deterministically includes the generating $\bxi$ in its sample batch, but does not account for this when computing the SNIS weights. Although \citet{xu2023stable} prove that as $n \rightarrow \infty$, the bias of STF converges to 0, we see that the bias of the estimator is not negligible for a practical choice of $n$.

Finally, we find that our estimator significantly outperforms even a near-SoTA diffusion model for most $t$. Surprisingly, we find that peak MSE for the diffusion model does not coincide with the Monte Carlo methods, and instead occurs at substantially lower $t$ than the peak for both SNIS estimators.  

We further compare our method and EDM qualitatively in \cref{fig:denoiser_viz}. We find our method better estimates the posterior mean for most $t$. However, we find some evidence of artifacts at high $t$ which may be mitigated by increasing $n$.
{\section{Consistency Training} \label{sec:consistency}}
    \subsection{Consistency Models}
We use our estimator to train consistency models on CIFAR-10 \cite{krizhevsky2009learning}. Specifically, we replace the single-sample score estimator used to produce $\hat{\bz}$ in \cref{eq:CM} with SNIS score estimators. We use the improved consistency training (iCT) techniques proposed by \citet{song2023improved}, comparing a baseline iCT model with models trained with an STF \cite{xu2023stable} and a KNN score estimator. For both SNIS score estimators, we use $n=256$, with a $k=2056$ for the KNN proposal. We measure performance using Frechet Inception Distance \cite{heusel2017gans} and Inception Score \cite{salimans2016improved}.

\begin{table}[t!]
    \centering
    \captionsetup{skip=5pt}
    \caption{Single step unconditional image generation performance on CIFAR-10. Bold indicates the best result within a section, underline indicates the best result overall.}
    \begin{tabularx}{\linewidth}{L r  r }
        \hline
         \textbf{Method}&  \textbf{FID}$\downarrow$&  \textbf{IS}$\uparrow$ \\
         \hline
         % PD \cite{salimans2022progressive} & 8.34 & 8.69 \\
         CD \cite{song2023consistency} & 3.55 & 9.48 \\
         % CTM* \cite{kim2023consistency} & \textbf{\underline{1.98}} & - \\
         CT \cite{song2023consistency} & 8.70 & 8.49 \\
         iCT \cite{song2023improved} & \textbf{\underline{2.83}} & \textbf{\underline{9.54}} \\
         % CTM* \cite{kim2023consistency} & \textbf{2.39} & - \\
         \hline
         % iCT (Reimplemented, batch=512) & 4.17 & 9.24 \\
         iCT (Reimplemented) & 3.67 & 9.45 \\
         iCT + STF \cite{xu2023stable} & 3.79 & \textbf{9.50} \\
         iCT + KNN (ours) & \textbf{3.59} & 9.49 \\
         \hline
    \end{tabularx}
    \label{tab:cifar_fid}
\end{table}
From \cref{tab:cifar_fid}, we observe consistency training with KNN score estimators results in lower FID and higher Inception Score compared to our re-implementation of iCT \cite{song2023improved}. While we are unable to reproduce the results of \citep{song2023improved} outright (training code was not available at the time of writing) we do improve over our re-implementation of iCT. Surprisingly, we find that using the STF score estimator \cite{xu2023stable} results in worse FID than both our method and the iCT baseline. We hypothesize STF's bias may contribute to this degradation. 
\begin{figure}[t!]
    \centering
    \includegraphics[width=\linewidth]{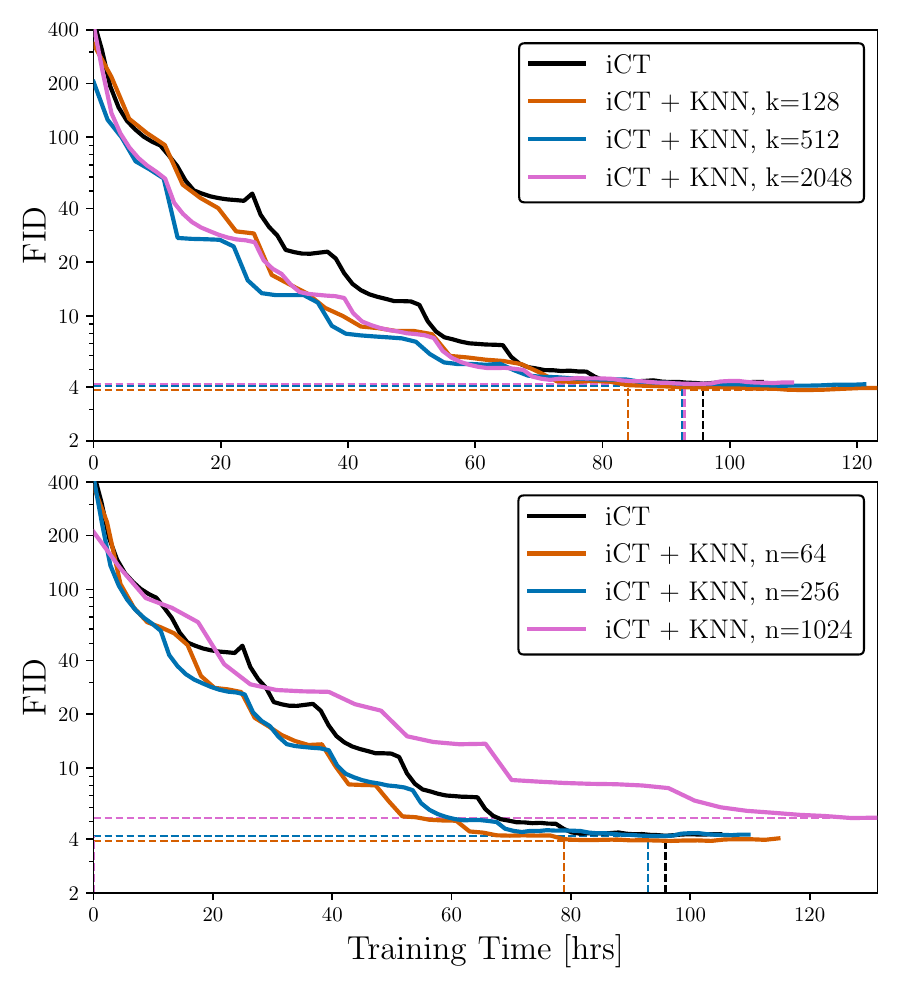}
    \caption{Impact of KNN score estimator performance on Consistency Training. Horizontal lines indicate minimum FID per model, vertical lines indicate when FID improves on baseline iCT. \textbf{Top:} Effect of varying KNN search size. \textbf{Bottom:} Effect of varying estimator sample size.}
    \label{fig:ablation}
\end{figure}
\subsection{Ablations}
We sweep over $k$ and $n$ to investigate their impact on downstream model performance. For our ablation, we train consistency models with the same number of iterations but halve the batch size to reduce computational cost.

\Cref{fig:ablation} highlights that for most hyperparameter choices, our method improves the convergence rate of consistency training. KNN models generally reach lower FID in less time than the iCT baseline. However, the training time required to achieve optimal FID is similar across all methods. Notably, we find reducing $n$ and $k$ improves sample FID, even though this increases estimator variance (\cref{ap:estimator_ablation}). We suspect score variance may have a regularizing effect, and suggest future work investigate further hyperparameter tuning to achieve optimal performance.
\section{Diffusion Sampling}
    \begin{figure}
    \centering
    \includegraphics[width=\linewidth]{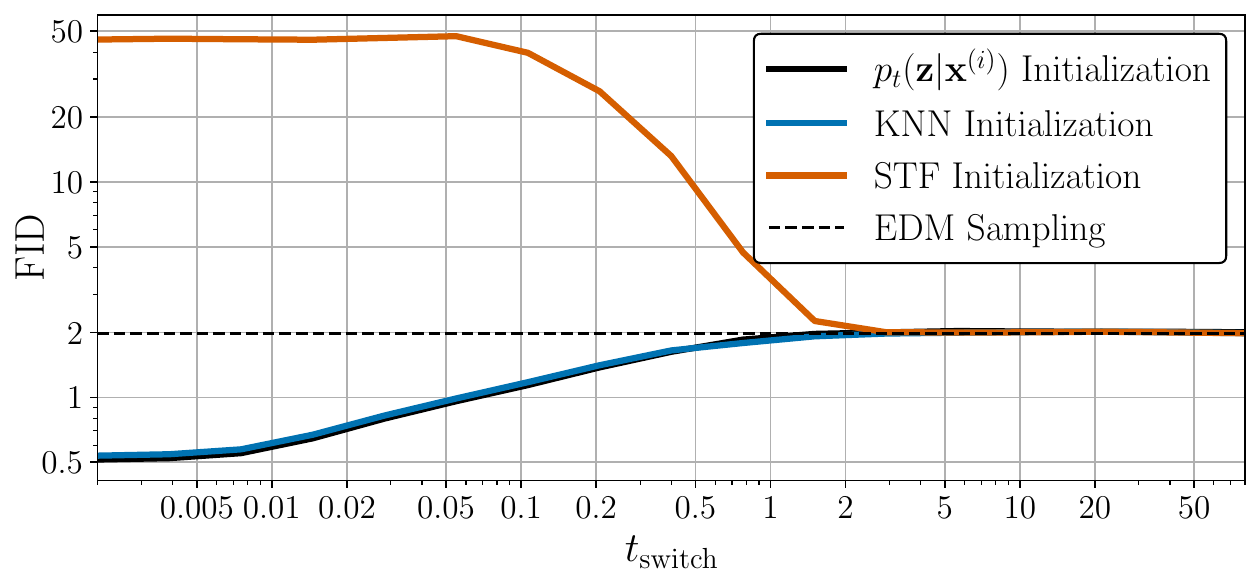}
    \caption{Comparison of sample FID versus $t_{\mathrm{switch}}$ in our hybrid sampling approach. Initializing diffusion sampling with KNN PF-ODE integration ith our score estimator yields identical performance to forward process initialization. For $t<2$, STF is unsuitable for PF-ODE integration.}
    \label{fig:diffusion_init}
\end{figure}
Motivated by the low error of our estimator compared to network approximators (\cref{fig:estimator_perf}) we investigate the effect of replacing a diffusion model with our estimator for diffusion sampling. To this end, we generate samples via PF-ODE integration using a hybrid sampling approach. For $T < t < t_{\mathrm{switch}}$, we integrate \cref{eq:pfode} with our score estimate using the Huen solver proposed by \citet{karras2022elucidating}. After $t_{\mathrm{switch}}$, we finish sampling with a pre-trained EDM score estimator.

\Cref{fig:diffusion_init} plots the quality of the resulting hybrid samples against $t_{\mathrm{switch}}$. We compare our estimator, STF, and a baseline which initializes EDM sampling using samples from the forward likelihood. Naively, our estimator is able to drive FID to an unprecedented minimum of 0.5. However, this performance is matched by forward process initialization, indicating that integrating the PF-ODE with our estimator is equivalent to drawing samples from the training dataset. This means our method cannot be used on its own as a generative model, as it does not produce generalization.

As shown by \cref{sec:consistency} and by \citet{xu2023stable}, training diffusion generative models with improved score estimates does not degrade generalization. However, when sampling using an estimator which closely matches the \emph{exact minimizer} of the diffusion training objective (\cref{eq:DSM}), we see that generalization disappears. This reinforces the findings of \cite{yi2023generalization}: that the lauded generalization properties of diffusion models can be attributed to bias in the network score estimator caused by architecture, optimization procedure or other factors.

Although integrating the PF-ODE using our estimator does not produce generalization, we still closely match the forward process initialization, indicating that our method is capable of general PF-ODE integration. This finding suggests that using our estimator, it may be possible to convert distillation methods \cite{salimans2022progressive, sauer2023adversarial, lu2023cm} into from-scratch training procedures, opening promising avenues of future research.

In contrast, when $t_{\mathrm{switch}} < 2$, the FID of STF increases, indicating it is unsuitable for general purpose PF-ODE integration. We hypothesize that this increase is because for small $t$, the STF estimator's performance is reliant on including $\bxi$ used to generate $\bz$ in the SNIS batch. For general PF-ODE integration where there is no such $\bxi$, the uniform proposal of STF is unable to produce a suitable score estimate.

\section{Related Work}
    \textbf{Generative Modelling with Diffusion Score Estimators}
Diffusion models \cite{sohl2015deep, ho2020denoising, song2020score} have become a widely utilized class of generative models which use learned score estimators to generate data across a variety of fields. Learned score estimators are also used to train downstream models with distillation \cite{salimans2022progressive}, consistency matching \cite{song2023consistency, kim2023consistency}, or to regularize the training of generative adversarial models \cite{sauer2023adversarial,lu2023cm}. 

\textbf{Importance Sampling for Generative Model Training}
Multiple methods have used importance sampling to reduce variance in generative model training. Reweighted Wake-Sleep \cite{bornschein2014reweighted} uses importance sampling to improve variance in the sleep-phase of the wake-sleep algorithm \cite{hinton1995wake}. IWAE \cite{burda2015importance} uses importance sampling to tighten the ELBO used to train variational autoencoders \cite{kingma2013auto}.
Importance sampling has also been utilized to minimize bias in diffusion training \cite{kim2024training}. Similar to this work, Stable Target Fields \cite{xu2023stable} utilizes an SNIS estimator to estimate the marginal score of diffusion processes. However, their work uses a uniform proposal and is focused on improving diffusion models training instead of generally estimating the score function.

\textbf{Retrieval Augmented Methods}
Many recent methods have proposed using memory systems to improve generative models. In natural language processing, methods propose combining language model outputs \cite{khandelwal2019generalization, khandelwal2020nearest} or to augment transformer contexts  with results from a similarity search based over an external database \cite{guu2020retrieval, borgeaud2022improving}. In computer vision, RetrieveGAN \cite{tseng2020retrievegan} conditions generated images on patches queried using text from a dataset while RetrievalFuse \cite{siddiqui2021retrievalfuse} searches over a database of 3D structures to augment scene reconstruction. In diffusion models, KNN-Diffusion \cite{sheynin2022knn} and Retrieval-Augmented Diffusion Models \cite{blattmann2022retrieval} both condition diffusion models on the k nearest CLIP \cite{radford2021learning} embeddings of clean training images. Although our method also uses a KNN search for diffusion generative models, our approach is orthogonal to these methods as we do not condition our network on the retrieved images.
\section{Conclusions}
    Our work introduces a nearest neighbour estimator of the score function for diffusion generative models. We analytically bound the variance of our estimator, and show empirically that it has low bias and variance compared to other estimation techniques. Training consistency models, we find that our estimator both increases training speed and sample quality. We highlight how our estimator can be used for general purpose probability flow ODE traversal. We believe our work opens intriguing opportunities for future research. One possible research direction is substituting our estimator for pre-trained diffusion networks in distillation approaches, to produces new from-scratch training methods. Another area of work is exploring alternate metric spaces. The $\ell_2$ distance which forms the basis of our method does not capture high level similarity between images well. By transforming our data into a latent space as described by \citet{rombach2022high}, we may be able leverage more informative neighbours for score estimation. Finally, we believe more research is warrented to investigate the cause of observed differences between peak score network estimator error and the peak error of SNIS estimators.

\section*{Impact Statement}
    Diffusion models require large amounts of computation and energy to train. Arguably, generating from such models is even more energy intensive. Consistency models address this inference energy problem by limiting the number of integration steps required to generate data. Directly training consistency models, rather than learning from a pre-trained diffusion model further lessons the energy burden. Our work directly addresses all these energy problems by making it possible to train better consistency models more energy efficiently, without first training a diffusion model. Powerful generative AI models are dual-use technology. We generally  believe that democratizing access to the technology by lowering energy requirements for training and inference helps rather than harms.
% TODO: Write replace boilerplate impact statement.
%This paper presents work whose goal is to advance the field of Machine Learning. There are many potential societal consequences of our work, none which we feel must be specifically highlighted here.

% Another possible avenue for impact is that this may help with the commonly cited problem of memorization in generative models by identifying which elements of the train set are being used to estimate the score
\section*{Acknowledgements}
We acknowledge the support of the Natural Sciences and Engineering Research Council of Canada (NSERC), the Canada CIFAR AI Chairs Program, Inverted AI, MITACS, the Department of Energy through Lawrence Berkeley National Laboratory, and Google. This research was enabled in part by technical support and computational resources provided by the Digital Research Alliance of Canada Compute Canada (alliancecan.ca), the Advanced Research Computing at the University of British Columbia (arc.ubc.ca), Amazon, and Oracle.

\bibliography{refs}
\bibliographystyle{icml2024}

%%%%%%%%%%%%%%%%%%%%%%%%%%%%%%%%%%%%%%%%%%%%%%%%%%%%%%%%%%%%%%%%%%%%%%%%%%%%%%%
%%%%%%%%%%%%%%%%%%%%%%%%%%%%%%%%%%%%%%%%%%%%%%%%%%%%%%%%%%%%%%%%%%%%%%%%%%%%%%%
% APPENDIX
%%%%%%%%%%%%%%%%%%%%%%%%%%%%%%%%%%%%%%%%%%%%%%%%%%%%%%%%%%%%%%%%%%%%%%%%%%%%%%%
%%%%%%%%%%%%%%%%%%%%%%%%%%%%%%%%%%%%%%%%%%%%%%%%%%%%%%%%%%%%%%%%%%%%%%%%%%%%%%%
\newpage
\appendix
\onecolumn
\section*{Appendices}
\section{Derivations}
{\subsection{Marginal Score as a Posterior Expectation of Conditional Scores} \label{ap:cond_expecation}}

Below, we derive \cref{eq:app_marginal_score_cond}.

\begin{align}
    \nabla_{\bz} \log p_t(\bz) &= \frac{\nabla_{\bz} p_t(\bz)}{p_t(\bz)} \\
    &= \frac{\nabla_\bz \int p_t(\bz | \bx) p(\bx) d\bx}{p_t(\bz)} \\
    &= \frac{\int \nabla_\bz p_t(\bz | \bx) p(\bx) d\bx}{p_t(\bz)}  \label{eq:leibniz} \\
    &= \frac{\int \frac{p_t(\bz | \bx)}{p_t(\bz | \bx)} \nabla_\bz p_t(\bz | \bx) p(\bx) d\bx}{p_t(\bz)}\\
    &= \frac{\int p_t(\bz | \bx) \nabla_\bz \log p_t(\bz | \bx) p(\bx) d\bx}{p_t(\bz)} \\
    &= \int \frac{p_t(\bz | \bx) p(\bx)}{p_t(\bz)} \nabla_\bz \log p_t(\bz | \bx) d\bx \\
    &= \int p_t(\bx | \bz) \nabla_\bz \log p_t(\bz | \bx) d\bx  \label{eq:bayes} \\ 
    &= \mathop{\mathbb{E}}_{\bx \sim p_t(\bx | \bz)}\left[\nabla_\bz \log p_t(\bz | \bx)\right].  \label{eq:app_marginal_score_cond}
\end{align}

\cref{eq:leibniz} is due to the Leibniz integral rule, while \cref{eq:bayes} is an application of Bayes Rule.

\subsection{Conditional Score}

 We define $p_t(\bz | \bx) = \mathcal{N}(\bz; \bx, t^2 \mathbf{I}).$ Then the conditional score of $p_t(\bz | \bx)$ is

\begin{align}
    \nabla_\bz \log p_t(\bz | \bx) &= \frac{\nabla_\bz p_t(\bz | \bx)}{p_t(\bz | \bx)} \\
    &= \frac{1}{p_t(\bz | \bx)} \cdot \nabla_{\bz}\frac{1}{\sqrt{(2\pi)^d \text{det}\left(t^2\mathbf{I}\right)}} \text{exp}\left(\frac{-1}{2t^2}\left(\bz - \bx\right)^\top\left(\bz - \bx\right)\right)\\
    &= \frac{p_t(\bz | \bx)}{p_t(\bz | \bx)}\nabla_{\bz}\frac{-1}{2t^2}\left(\bz - \bx\right)^\top\left(\bz - \bx\right) \\
    &= \frac{\bx - \bz}{t^2}. \label{eq:app_cond_score}
\end{align} 

\subsection{Marginal Score via Posterior Mean}

Substituting \cref{eq:app_cond_score} into \cref{eq:app_marginal_score_cond} we get

\begin{align*}
    \nabla_\bz \log p_t(\bz) &= \mathop{\mathbb{E}}_{\bx \sim p_t(\bx | \bz)}\left[ \frac{\bx - \bz}{t^2}\right] \\
    &= \frac{\mathbb{E}_{\bx \sim p_t(\bx | \bz)}\left[\bx \right] - \bz}{t^2} \\
    &= \frac{\mathbb{E}\left[\bx | \bz\right] - \bz}{t^2},
\end{align*}

where $\mathbb{E}[\bx | \bz] = \mathop{\mathbb{E}}_{\bx \sim p_t(\bx | \bz)} \left[\bx\right]$. 

{\subsection{Score Matching Optimal Denoiser} \label{ap:minimizer}}

We consider the diffusion objective from \cref{eq:DSM},
\begin{align}
     \mathcal{L}_t &= \mathbb{E}_{\bz, \bxi \sim p_t(\bz, \bxi)}\left[\lVert \mathbf{D}_\theta\left(\mathbf{z}, t\right) - \nabla_\bz \log p_t\left(\bz | \bx\right) \rVert_2^2 \right] \\
     &= \mathbb{E}_{\bz, \bx \sim p_t(\bz, \bx)}\left[ \mathbf{D}_\theta(\bz, t)^2 - 2\mathbf{D}_\theta(\bz, t) \bxi + (\bxi)^2 \right].\\
\end{align}
For a fixed $\bz$, we can write
\begin{align}
     \mathcal{L}_{t,\bz} &= \mathbb{E}_{\bxi \sim p_t(\bxi | \bz)}\left[ \mathbf{D}_\theta(\bz, t)^2 - 2\mathbf{D}_\theta(\bz, t)\bxi + (\bxi)^2 \right] \\
     &= \mathbf{D}_\theta(\bz, t)^2  - 2\mathbf{D}_\theta(\bz, t) \mathbb{E}_{\bxi \sim p_t(\bxi | \bz)}\left[\bxi\right]  + \mathbb{E}_{\bxi \sim p_t(\bxi | \bz)}\left[ (\bxi)^2 \right]. \label{eq:DSM_expanded}
\end{align}
\cref{eq:DSM_expanded} is quadratic with respect to $\mathbf{D}_\theta(\bz, t)$, so a unique minimizer $\mathbf{D}_\theta^\star(\bz, t)$ can be found by solving $\frac{\partial \mathcal{L}}{\partial \mathbf{D}_\theta(\bz, t)} = \mathbf{0}.$

\begin{align}
    \frac{\partial \mathcal{L}_{t,z}}{\partial \mathbf{D}_\theta(\bz, t)} = \mathbf{0} &= 2\mathbf{D}_\theta(\bz, t) -2 \mathbb{E}_{\bxi \sim p_t(\bxi | \bz)}\left[\bxi \right] \\
    \mathbf{D}_\theta^\star(\bz, t) &= \mathbb{E}_{\bxi \sim p_t(\bxi | \bz)}\left[\bxi\right]\\
    &= \pmean.
\end{align}

\subsection{Consistency Training as Consistency Matching with Euler Solver and One Sample Score Estimator}

The original Consistency Training objective from \cite{song2023consistency} is formulated as

\begin{equation}
    \mathcal{L}_{CT}^N = \mathbb{E}\left[\lambda(t_{n}) d\left(\student\left(\mathbf{x} + t_{n}\cdot\bm{\epsilon}, t\right), \teacher\left(\mathbf{x} + t_{n-1}\cdot \bm{\epsilon}, t_{n-1}\right)\right) \right], \label{eq:CT_original}
\end{equation} 

where $\bm{\epsilon} \sim \mathcal{N}(\mathbf{0}, \mathbf{I})$. Via the reparameterization trick, we define $\bz \sim p_{t_n}(\bz | \bx)$ as

\begin{equation}
    \bz = \bx + t_n \bm{\epsilon} \label{eq:reparam}.
\end{equation}

Substituting \cref{eq:reparam} into the Euler update function $\texttt{solver}(\bz, t_n, t_{n-1}, \nabla_\bz p_{t_n}(\bz, t_n)) = \bz + (t_{n-1} - t_n)\frac{d\bz}{d t}$ for \cref{eq:pfode} with the single sample score estimate % Todo: Add E[x|z] = x. 
$\nabla_\bz \log p_t(\bz) \approx \frac{\bx - \bz}{t^2}$ yields

\begin{align}
    \bz' &= \bz - (t_{n-1} - t_n) \frac{\bx - \bz}{t_n} \\
    &= \bx + t_n \bm{\epsilon} - (t_{n-1} - t_n) \frac{\bx - \bx + t_n \epsilon}{t_n} \\
    &= \bx + t_n \bm{\epsilon} - t_n \bm{\epsilon} + t_{n-1} \bm{\epsilon} \\
    &= \bx t_{n-1} \bm{\epsilon}. \label{eq:z_prime}
\end{align}

Substituting \cref{eq:reparam} and \cref{eq:z_prime} into \cref{eq:CT_original} yields

\begin{equation}
    \mathcal{L}_{CT}^N = \mathbb{E}\left[\lambda(t_{n}) d\left(\student\left(\bz, t\right), \teacher\left(\bz', t_{n-1}\right)\right) \right]
\end{equation}

which matches the form in \cref{eq:CM}.
{\section{Extension to Generalized Diffusion Processes} \label{ap:general_diffusion}}
For simplicity, this paper utilizes the EDM diffusion process given by \cref{eq:pfode}. However, other diffusion processes are widely used which vary the rate at which noise is added to the data and the scaling of the data through the diffusion process. \citet{karras2022elucidating} introduce a general purpose PF ODE which captures these choices
\begin{equation}
    d\bx = \left[\frac{\Dot{s}(t)}{s(t)}\bx - s(t)^2\Dot{\sigma}(t)\sigma(t) \nabla\bz \log p_t\left(\frac{\bz}{s(t)}\right)\right]dt. \label{eq:general_pfode}
\end{equation}
In our work, we select $\sigma(t) = t$, $s(t) = 1$, however other choices are possible. For example, one popular choice is the variance preserving diffusion process \cite{ho2020denoising, song2020score}, corresponding to 
\begin{align}
    \sigma(t) &= \sqrt{e^{\frac{1}{2}\beta_d t^2 + \beta_{min}t} - 1}\\
    s(t) &= 1 / \sqrt{e^{\frac{1}{2}\beta_d t^2 + \beta_{min}t}}.
\end{align}
The derivations of $s(t)$ and $\sigma(t)$ is covered in depth in \cite{karras2022elucidating}.
We now derive a general KNN score estimator for a generic diffusion processes defined by $s(t)$ and $\sigma(t)$. First, we define the forward likelihood for the generalized diffusion process.
\begin{equation}
    p_t(\bz | \bxi) = \mathcal{N}\left(\bz; s(t)\bxi, s(t)^2\sigma(t)^2\mathbf{I}\right). \label{eq:general_forward}
\end{equation}
Since the score function which is used for evaluation of \cref{eq:general_pfode} is a function of $\bz / s(t)$, we express \cref{eq:general_forward} in that manner
{ \allowdisplaybreaks
\begin{align}
    p_t(\bz | \bxi) &= \mathcal{N}\left(\bz; s(t)\bxi, s(t)^2\sigma(t)^2\mathbf{I}\right) \\
    &= (2\pi)^{\frac{-d}{2}} s(t)^{-d}\sigma(t)^{-d}\text{exp}\left(\frac{-\left \lVert \bz - s(t)\bxi \right \rVert_2^2}{2s(t)^2\sigma(t)^2}\right) \\
    &= (2\pi)^{\frac{-d}{2}} s(t)^{-d}\sigma(t)^{-d}\text{exp}\left(\frac{-s(t)^2\left \lVert \frac{\bz}{s(t)} - \bxi \right \rVert_2^2}{2s(t)^2\sigma(t)^2}\right) \\
    &= s(t)^{-d} \cdot (2\pi)^{\frac{-d}{2}} \sigma(t)^{-d}\text{exp}\left(\frac{-\left \lVert \frac{\bz}{s(t)} - \bxi \right \rVert_2^2}{2\sigma(t)^2}\right) \\
    &= s(t)^{-d} \mathcal{N}\left(\frac{\bz}{s(t)}; \bxi, \sigma(t)^2\bm{I}\right).
\end{align}
}
We define 
\begin{equation}
    p_t\left(\frac{\bz}{s(t)} \Big| \bxi\right) = \mathcal{N}\left(\frac{\bz}{s(t)}; \bxi, \sigma(t)^2\bm{I}\right). \label{eq:scaled_likelihood}
\end{equation}
The score of \cref{eq:scaled_likelihood} is
\begin{align}
    \nabla_{\bz} \log p_t\left(\frac{\bz}{s(t)} \Big| \bxi\right) &= p_t\left(\frac{\bz}{s(t)} \Big| \bxi\right)^{-1}\nabla_{\bz}p_t\left(\frac{\bz}{s(t)} \Big| \bxi\right) \\
    &= p_t\left(\frac{\bz}{s(t)} \Big| \bxi\right)^{-1} \cdot p_t\left(\frac{\bz}{s(t)} \Big| \bxi\right) \cdot \nabla_{\bz} \frac{-\left\lVert \frac{\bz}{s(t)} - \bxi \right \rVert_2^2}{2\sigma(t)^2} \\
    &= \frac{\bxi - \frac{\bz}{s(t)}}{s(t)\sigma(t)^2}.
\end{align}
Using \cref{eq:score_identity} derived in \cref{ap:cond_expecation}, we write
\begin{align}
    \nabla_{\bz} \log p_t\left(\frac{\bz}{\sigma(t)}\right) &= \mathop{\mathbb{E}}_{\bxi \sim p_t(\bxi | \bz / s(t))} \left[\nabla_{\bz} \log p_t\left(\frac{\bz}{s(t)} \Big| \bxi\right) \right] \\
    &= \mathop{\mathbb{E}}_{\bxi \sim p_t(\bxi | \bz / s(t))} \left[\frac{\bxi - \frac{\bz}{s(t)}}{s(t)\sigma(t)^2} \right] \\
    &= \frac{\mathop{\mathbb{E}}\left[\bx \Big| \frac{\bz}{s(t)}, t\right] - \frac{\bz}{s(t)}}{s(t)\sigma(t)^2}.
\end{align}
Therefore, estimating the score function in \cref{eq:general_pfode} is equivalent to estimating the posterior mean $\mathop{\mathbb{E}}\left[\bx \Big| \frac{\bz}{s(t)}, t\right]$. The posterior distribution $p_t\left(\bxi \Big| \frac{\bz}{s(t)}\right)$ can be expressed using Bayes Rule as
\begin{align}
    p_t\left(\bxi \Big| \frac{\bz}{s(t)}\right) &= \frac{p_t\left( \frac{\bz}{s(t)}\Big| \bxi \right)p\left(\bxi\right)}{p_t\left(\frac{\bz}{s(t)}\right)} \\
    &= \frac{p_t\left( \frac{\bz}{s(t)}\Big| \bxi \right)}{N p_t\left(\frac{\bz}{s(t)}\right)}.
\end{align}
We define $\mathcal{K}'$ to be the $k$ most probable elements of $p_t\left(\bxi \Big| \frac{\bz}{s(t)}\right)$ let $\bx_{k}$ be the element of $\mathcal{K}'$ with the smallest posterior probability. Then, using $\mathcal{K}'$, we define a general proposal
\begin{equation}
    \mathcal{Z}_q \cdot q_t\left(\bxi \Big | \frac{\bz}{s(t)}\right) = \begin{cases}
        & p_t\left( \frac{\bz}{s(t)}\Big| \bxi \right) \ \forall \ \bxi \in \mathcal{K}'\\
        & p_t\left( \frac{\bz}{s(t)}\Big| \bx_k \right) \ \forall \ \bxi \notin \mathcal{K}'.\\
    \end{cases} \label{eq:general_proposal}
\end{equation}
Because $p_t\left(\bxi  \Big| \frac{\bz}{s(t)} \right) \propto p_t\left( \frac{\bz}{s(t)}\Big| \bxi \right)$ and $p_t\left( \frac{\bz}{s(t)}\Big| \bxi \right)$ is Gaussian, the most probable elements of $p_t\left(\bxi  \Big|\frac{\bz}{s(t)}  \right) $ are the elements with the smallest distance $\left \lVert \frac{\bz}{s(t)} - \bxi \right \rVert_2^2$. We can therefore use a fast KNN search with query $\frac{\bz}{s(t)}$ over $\mathcal{D}$ to identify $\mathcal{K}'$. Together, an algorithm for estimating the score of a generic diffusion process is outlined in \cref{alg:general_estimator}

\begin{algorithm}[h!]
   \caption{General Nearest Neighbour Score Estimator}
   \label{alg:general_estimator}
\begin{algorithmic}
    \STATE {\bfseries Input:} $\bz$, dataset $\mathcal{D}$, index $\mathcal{I}$, neighbour size $k$, sample size $n$, scale $s(t)$, noise level $\sigma(t)$
    % Get neighbours of z
    \STATE $\mathcal{K}', d \leftarrow \texttt{search}(I, \bz / s(t), k)$
    \STATE $p_t(\frac{\bz}{s(t)} | \bx_i) \leftarrow \text{exp}(\frac{-d_i^2}{2\sigma(t)^2})  \ \forall \ \bx_i \in \mathcal{K}'$ 
    \STATE $\mathcal{Z}_q \leftarrow \sum_{i=1}^k p_t(\frac{\bz}{s(t)}  | \bx_i) + (N-k)p_t(\frac{\bz}{s(t)}  | \bx_k)$
    \STATE $q_t(\bx^{(i)} | \frac{\bz}{s(t)} ) \leftarrow \frac{1}{\mathcal{Z}_q } \begin{cases}
        &p_t (\frac{\bz}{s(t)}  | \bx_i)\ \forall \ \bxi \in \mathcal{K}' \\
        & p_t (\frac{\bz}{s(t)} | \bx_k)\ \forall \ \bxi \notin \mathcal{K}'
    \end{cases} $
    \STATE Sample $\left\{\bx_1, \ldots, \bx_n \right\} \sim q_t(\bxi | \frac{\bz}{s(t)})$
    \STATE $w_i = p(\frac{\bz}{s(t)} | \bx_i) / q_t(\bx_i | \frac{\bz}{s(t)})$
    \STATE $\bar{w}_i = \frac{w_i}{\sum_{j=1}^n w_j}$
    \STATE $\hat{\bx} = \sum_{i=1}^n \bar{w}_i \bx_i$ \hfill\COMMENT{\Cref{eq:SNIS_estimator}}
    \STATE \textbf{return} $\frac{\hat{\bx} - \frac{\bz}{s(t)}}{s(t)\sigma(t)^2}$  
\end{algorithmic}
\end{algorithm}

% \begin{algorithm}[t!]
%    \caption{Nearest Neighbour Score Estimator}
%    \label{alg:estimator}
% \begin{algorithmic}
%     \STATE {\bfseries Input:} $\bz$, $t$, $\mathcal{D}$,  $\mathcal{I}$,$k$,$n$
%     % Get neighbours of z
%     \STATE $\mathcal{K}, d \leftarrow \texttt{search}(\mathcal{I}, \bz, k)$
%     \STATE $p_t(\bz | \bx_i) \leftarrow \text{exp}(\frac{-d_i^2}{2t^2}) \ \forall \ \bx_i \in \mathcal{K}$ 
%     % \STATE $q_t(\bx^{(i)} | \bz) \leftarrow \text{exp}(\frac{-d_k^2}{2t^2}) \ \forall \ \bxi \notin \mathcal{K}$
%     \STATE $\mathcal{Z}_q \leftarrow \sum_{i=1}^k p_t(\bz | \bx_i) + (N-k)p_t( \bz | \bx_k)$
%     \STATE $q_t(\bx^{(i)} | \bz) \leftarrow \frac{1}{\mathcal{Z}_q } \begin{cases}
%         &p_t (\bz | \bx_i)\ \forall \ \bxi \in \mathcal{K} \\
%         & p_t (\bz | \bx_k)\ \forall \ \bxi \notin \mathcal{K}
%     \end{cases} $
%     \STATE Sample $\left\{\bx_1, \ldots, \bx_n \right\} \sim q_t(\bxi | \bz)$
%     \STATE $w_i \leftarrow p(\bz | \bx_i) / q_t(\bx_i | \bz)$
%     \STATE $\bar{w}_i \leftarrow \frac{w_i}{\sum_{j=1}^n w_j}$
%     \STATE $\hat{\mu}_{\text{KNN}} \leftarrow \sum_{i=1}^n \bar{w}_i \bx_i$
%     \STATE \textbf{return} $\frac{\hat{\mu}_{\text{KNN}} - \bz}{t^2}$  
%     % Build proposal
%     % Sample from proposal
%     % Compute proposal weights
%     % Normalize proposal weights
%     % return Weighted average
% \end{algorithmic}
% \end{algorithm}
\section{Proofs}
% {\subsection{KNN estimator variance lower bound} \label{ap:knn_lower}}
% Consider the case when $q_t\left(\bxi | \bz\right) = p_t(\bxi | \bz)$ In this case we have
% \begin{align}
%     n \cdot \text{Var}\left(\snispmean\right) &= \sum_{i=1}^{N} \frac{p_t\left(\bxi | \bz\right)^2}{q_t\left(\bxi | \bz\right)}\left(\bxi - \pmean\right)^{\circ2} \\
%     n \cdot \text{Var}\left(\snispmean\right) &= \sum_{i=1}^{N} \frac{p_t\left(\bxi | \bz\right)}{p_t\left(\bxi | \bz\right)} p_t(\bxi | \bz)\left(\bxi - \pmean\right)^{\circ2} \\
%     n \cdot \text{Var}\left(\snispmean\right) &= \sum_{i=1}^{N}p_t(\bxi | \bz)\left(\bxi - \pmean\right)^{\circ2} \\
%     n \cdot \text{Var}\left(\snispmean\right) &= n \cdot \text{Var} \left(\hat{\mu}_{\text{MC}}\right) \\  
%     \text{Var}\left(\snispmean\right) &= \text{Var} \left(\hat{\mu}_{\text{MC}}\right) \\  
% \end{align}
% In this case, the variance of the KNN posterior mean estimator is the same as a simple Monte Carlo estimator with samples drawn from the posterior
{\subsection{Proof of \cref{thm:knn_bound}} \label{ap:knn_upper}}

\Cref{thm:knn_bound} states\\

Let $t \in (0, \infty)$, then for a fixed $n$
\begin{align}
   % \text{Tr}\left(\text{Cov}\left(\hat{\mu}_{\text{KNN}}\right)\right) &\geq  \text{Tr}\left(\text{Cov}\left(\hat{\mu}_{\text{MC}}\right) \right)\\
    \text{Tr}\left(\text{Cov}\left(\hat{\mu}_{\text{KNN}}\right)\right)\leq
    \frac{\mathcal{Z}_q}{p_t(\bz)}\text{Tr}\left(\text{Cov}\left(\hat{\mu}_{\text{MC}}\right)\right).
\end{align}
\begin{proof}
Starting with the diagonal of the covariance of the KNN estimator, we have
{\allowdisplaybreaks
\begin{align}
    &n \cdot \text{Diag}\left(\text{Cov}\left(\snispmean\right)\right) = \sum_{i=1}^{N} \frac{p_t\left(\bxi | \bz\right)^2}{q_t\left(\bxi | \bz\right)}\left(\bxi - \pmean\right)^{\circ2} \\
    &= \sum_{\bxi \in \mathcal{K}}\frac{p_t\left(\bxi | \bz\right)^2}{q_t\left(\bxi | \bz\right)}\left(\bxi - \pmean\right)^{\circ2} + \sum_{\bxi \notin \mathcal{K}} \frac{p_t\left(\bxi | \bz\right)^2}{q_t\left(\bxi | \bz\right)}\left(\bxi - \pmean\right)^{\circ2} \\
    \begin{split}
        &= \frac{\mathcal{Z}_q}{p_t\left(\bz\right)}\sum_{\bxi \in \mathcal{K}}\frac{p_t\left(\bz | \bxi\right)p\left(\bxi\right)}{p_t\left(\bz | \bxi\right)p\left(\bxi\right)} p_t\left(\bxi | \bz\right)\left(\bxi - \pmean\right)^{\circ2} \\ & \quad \quad \quad \quad + \frac{\mathcal{Z}_q}{p_t\left(\bz\right)}\sum_{\bxi \notin \mathcal{K}} \frac{p_t\left(\bz | \bxi\right)p\left(\bxi\right)}{p_t\left(\bz | \bx^{(k)}\right)p\left(\bxi\right)} p_t\left(\bxi | \bz\right)\left(\bxi - \pmean\right)^{\circ2}
    \end{split} \\
    \begin{split}
        &= \frac{\mathcal{Z}_q}{p_t\left(\bz\right)} \Bigg(\sum_{\bxi \in \mathcal{K}} p_t\left(\bxi | \bz\right)\left(\bxi - \pmean\right)^{\circ2}\\
        & \quad \quad \quad \quad + \sum_{\bxi \notin \mathcal{K}} \frac{p_t\left(\bz | \bxi\right)}{p_t\left(\bz | \bx^{(k)}\right)}p_t\left(\bxi | \bz\right)\left(\bxi - \pmean\right)^{\circ2} \\ 
        & \quad \quad \quad \quad  \pm \sum_{\bxi \notin \mathcal{K}}p_t\left(\bxi | \bz\right)\left(\bxi - \pmean\right)^{\circ2} \Bigg)
    \end{split} \\
    \begin{split}
        &= \frac{\mathcal{Z}_q}{p_t\left(\bz\right)} \Bigg(\sum_{i=1}^N p_t\left(\bxi | \bz\right)\left(\bxi - \pmean\right)^{\circ2}  \\
        & \quad \quad \quad \quad  + \sum_{\bxi \notin \mathcal{K}} \frac{p_t\left(\bz | \bxi\right)}{p_t\left(\bz | \bx^{(k)}\right)}p_t\left(\bxi | \bz\right)\left(\bxi - \pmean\right)^{\circ2} \\ 
        & \quad \quad \quad \quad  - \sum_{\bxi \notin \mathcal{K}}p_t\left(\bxi | \bz\right)\left(\bxi - \pmean\right)^{\circ2} \Bigg)
    \end{split} \\
    \begin{split}
        &= \frac{\mathcal{Z}_q}{p_t\left(\bz\right)} \Bigg(\sum_{i=1}^N p_t\left(\bxi | \bz\right)\left(\bxi - \pmean\right)^{\circ2}  \\
        & \quad \quad \quad \quad  + \sum_{\bxi \notin \mathcal{K}} \frac{p_t\left(\bz | \bxi\right)}{p_t\left(\bz | \bx^{(k)}\right)}p_t\left(\bxi | \bz\right)\left(\bxi - \pmean\right)^{\circ2} \\ 
        & \quad \quad \quad \quad  - \sum_{\bxi \notin \mathcal{K}}\frac{p_t\left(\bz | \bx^{(k)}\right)}{p_t\left(\bz | \bx^{(k)}\right)}p_t\left(\bxi | \bz\right)\left(\bxi - \pmean\right)^{\circ2} \Bigg)
    \end{split} \\
    \begin{split}
        &= \frac{\mathcal{Z}_q}{p_t\left(\bz\right)} \Bigg(\sum_{i=1}^N p_t\left(\bxi | \bz\right)\left(\bxi - \pmean\right)^{\circ2}  \\
        & \quad \quad \quad \quad  + \sum_{\bxi \notin \mathcal{K}} \frac{p_t\left(\bz | \bxi\right) - p_t\left(\bz | \bx^{(k)}\right) }{p_t\left(\bz | \bx^{(k)}\right)}p_t\left(\bxi | \bz\right)\left(\bxi - \pmean\right)^{\circ2}
    \end{split}\\
    \begin{split}
        &= \frac{\mathcal{Z}_q}{p_t\left(\bz\right)} \Bigg(\sum_{i=1}^N p_t\left(\bxi | \bz\right)\left(\bxi - \pmean\right)^{\circ2}  \\
        & \quad \quad \quad \quad  + \sum_{\bxi \notin \mathcal{K}} \frac{p_t\left(\bz | \bxi\right) - p_t\left(\bz | \bx^{(k)}\right) }{p_t\left(\bz | \bx^{(k)}\right)}p_t\left(\bxi | \bz\right)\left(\bxi - \pmean\right)^{\circ2}.
    \end{split}
\end{align}}
For $\bxi \notin \mathcal{K}, \ p_t\left(\bz | \bxi\right) \leq p_t\left(\bz | \bxk\right)$ so therefore $\left(p_t\left(\bz | \bxi\right) - p_t\left(\bz | \bxk \right)\right) \leq 0$. Since all other factors in the second term are positive, The term as a whole is negative. We upper bound this term with 0 and continue with an elementwise inequality on the elements of the vector
{\allowdisplaybreaks
\begin{align}
    n \cdot \text{Diag} \left(\text{Cov}\left(\snispmean\right)\right) &\leq \frac{\mathcal{Z}_q}{p_t\left(\bz\right)} \left(\sum_{i=1}^N p_t\left(\bxi | \bz\right)\left(\bxi - \pmean\right)^{\circ2} \right) \\
    &\leq \frac{\mathcal{Z}_q}{p_t\left(\bz\right)} \cdot n \cdot \text{Diag} \left(\text{Cov}\left(\hat{\mu}_{\text{MC}}\right)\right)\\
    \text{Diag}\left(\text{Cov}\left(\hat{\mu}_{\text{KNN}}\right) \right)&\leq \frac{\mathcal{Z}_q}{p_t\left(\bz\right)} \text{Diag} \left(\text{Cov}\left(\hat{\mu}_{\text{MC}}\right) \right).
\end{align}}
Let $\text{Diag}\left(\text{Cov}\left(\cdot\right) \right)^j$ denote the $j^{th}$ component of the diagonal vector. Since the inequaltiy is expressed element wise, $\forall j \in [0,\ldots, d]$ we have
\begin{equation}
\text{Diag}\left(\text{Cov}\left(\hat{\mu}_{\text{KNN}}\right) \right)^j\leq \frac{\mathcal{Z}_q}{p_t\left(\bz\right)} \text{Diag} \left(\text{Cov}\left(\hat{\mu}_{\text{MC}}\right) \right)^j.
\end{equation}
Summing over j, we write
\begin{align}
\sum_{j=1}^d\text{Diag}\left(\text{Cov}\left(\hat{\mu}_{\text{KNN}}\right) \right)^j&\leq \sum_{j=1}^d\frac{\mathcal{Z}_q}{p_t\left(\bz\right)} \text{Diag} \left(\text{Cov}\left(\hat{\mu}_{\text{MC}}\right) \right)^j\\
\sum_{j=1}^d\text{Diag}\left(\text{Cov}\left(\hat{\mu}_{\text{KNN}}\right) \right)^j&\leq \frac{\mathcal{Z}_q}{p_t\left(\bz\right)} \sum_{j=1}^d\text{Diag} \left(\text{Cov}\left(\hat{\mu}_{\text{MC}}\right) \right)^j\\
\text{Tr}\left(\text{Cov}\left(\hat{\mu}_{\text{KNN}}\right) \right)&\leq \frac{\mathcal{Z}_q}{p_t\left(\bz\right)} \text{Tr} \left(\text{Cov}\left(\hat{\mu}_{\text{MC}}\right) \right),
\end{align}
concluding the proof.
\end{proof}
% $\frac{\mathcal{Z}_q}{p_t(\bz)}$ is maximized when $k=1$. In this case, $q_t\left(\bx | \bz\right)$ is a uniform distribution over the dataset In that case: 
% \begin{align}
%     \frac{\mathcal{Z}_q}{p_t\left(\bz\right)} &= \frac{\sum_{\bxi \in \mathcal{K}}p_t(\bz | \bxi) + (N - k)p_t(\bz | \bxk)}{p_t(\bz)}\\
%     \frac{\mathcal{Z}_q}{p_t\left(\bz\right)} &= \frac{p_t(\bz | \bx^{(1)}) + (N -1) p_t(\bz | \bx^{(1)})}{p_t(\bz)}\\
%     \frac{\mathcal{Z}_q}{p_t\left(\bz\right)} &= \frac{N \cdot p_t(\bz | \bx^{(1)})}{p_t(\bz)}\\
%     \frac{\mathcal{Z}_q}{p_t\left(\bz\right)} &= \frac{N \cdot p_t(\bz | \bx^{(1)})}{\sum_{i=1}^N p_t(\bz | \bxi)} \\
%     \frac{\mathcal{Z}_q}{p_t\left(\bz\right)} &= \frac{N \cdot p_t(\bz | \bx^{(1)})}{p_t(\bz | \bx^{(1)}) + \sum_{i=2}^N p_t(\bz | \bxi)} \\
%     \frac{\mathcal{Z}_q}{p_t\left(\bz\right)} &\leq \frac{N \cdot p_t(\bz | \bx^{(1)})}{p_t(\bz | \bx^{(1)})} \\
%     \frac{\mathcal{Z}_q}{p_t\left(\bz\right)} &\leq N \\
% \end{align}
% Toghether, we have
% \begin{equation}
%     \text{Var}\left(\hat{\mu}_{\text{MC}}\right) \leq \text{Var}\left(\hat{\mu}_{\text{SNIS}}\right) \leq \frac{\mathcal{Z}_q}{p_t(\bz)} \cdot \text{Var}\left(\hat{\mu}_{\text{MC}}\right)
% \end{equation}
% Where $1 \leq \frac{\mathcal{Z}_q}{p_t(\bz)} \leq N$

\subsection{Uniform SNIS Trace of Covariance}
For utility in future proofs, we derive an expression for the trace of covariance for a self-normalized importance sampling estimator with a uniform proposal
\begin{lemma}
Let $\hat{\mu}_{\text{U}}$ be a SNIS posterior mean estimator with $q(\bxi) = \frac{1}{N} \ \forall \bxi \in D$ then,
\begin{equation}
      n \cdot \text{Tr}\left(\text{Cov}\left(\hat{\mu}_{\text{U}}\right) \right) = N \cdot \sum_{i=1}^N p_t(\bxi | \bz)^2(\bxi - \pmean)^{\circ 2}.  \label{eq:uniform_trace}
\end{equation}
\end{lemma}
\begin{proof}
   \begin{align}
    n \cdot \text{Diag}\left(\text{Cov}(\hat{\mu}_{\text{SNIS}})\right) &= \sum_{i=1}^N \frac{p_t(\bxi | \bz)^2}{q(\bxi)}(\bxi - \pmean)^{\circ 2} \\
    n \cdot \text{Diag}\left(\text{Cov}(\hat{\mu}_{\text{SNIS}})\right)  &= \sum_{i=1}^N \frac{p_t(\bxi | \bz)^2}{1/N}(\bxi - \pmean)^{\circ 2} \\
    n \cdot \text{Diag}\left(\text{Cov}(\hat{\mu}_{\text{SNIS}})\right)  &= N \cdot \sum_{i=1}^N p_t(\bxi | \bz)^2(\bxi - \pmean)^{\circ 2}.
\end{align} 
\end{proof}

% We note that for the case $p(\bxi | \bz) = \frac{1}{N} \ \forall \ \bxi \in \mathcal{D}$, we have
% \begin{align}
%     n \text{Var}(\hat{\mu}_{\text{U}}) &= N \cdot \sum_{i=1}^N \frac{p_t(\bxi | \bz)}{N}(\bxi - \pmean)^{\circ 2} \\
%     n \text{Var}(\hat{\mu}_{\text{U}}) &= \sum_{i=1}^N p_t(\bxi | \bz)(\bxi - \pmean)^{\circ 2} \\
%     n \text{Var}(\hat{\mu}_{\text{U}}) &= n \text{Var}(\hat{\mu}_{\text{MC}}) \\
%     \text{Var}(\hat{\mu}_{\text{U}}) &= \text{Var}(\hat{\mu}_{\text{MC}}) \label{eq:ap_stf_lower_bound}
% \end{align}
% For general case $p_t(\bxi | \bz)$, we have
% \begin{align}
%     n \text{Var}(\hat{\mu}_{\text{U}}) &= N \cdot \sum_{i=1}^N \frac{p_t(\bxi | \bz)}{N}(\bxi - \pmean)^{\circ 2} \\
%     n \text{Var}(\hat{\mu}_{\text{U}}) & \leq N \cdot \sum_{i=1}^N p_t(\bxi | \bz)(\bxi - \pmean)^{\circ 2} \label{eq:ap_stfleamm_ineq} \\
%     n \text{Var}(\hat{\mu}_{\text{U}}) & \leq N \cdot n \text{Var}(\hat{\mu}_{\text{MC}}) \\
%     \text{Var}(\hat{\mu}_{\text{U}}) & \leq N \cdot \text{Var}(\hat{\mu}_{\text{MC}}) \label{eq:ap_stf_upper_bound}
% \end{align}
% Where \cref{eq:ap_stfleamm_ineq} is because $p_t(\bxi | \bz) <=1 \ \forall \ \bxi \in \mathcal{D}$ with equality if $\exists \ \bxi \in \mathcal{D} \ s.t.\ p_t(\bxi | \bz) = 1$. Combining \cref{eq:ap_stf_lower_bound,eq:ap_stf_upper_bound}, we produce
% \begin{equation}
%     \text{Var}(\hat{\mu}_{\text{MC}}) \leq \text{Var}(\hat{\mu}_{\text{U}}) \leq N \text{Var}(\hat{\mu}_{\text{MC}}) \label{eq:ap_stf_var}
% \end{equation}

\subsection{Proof of \Cref{thm:knn_vs_stf}} \label{ap:knn_v_stf}
Let $t \in (0, \infty)$,  then for a fixed $n$
\begin{equation}
\begin{split}
    &\text{Tr}\left(\text{Cov}\left(\hat{\mu}_{\text{KNN}}\right) \right) \leq \\
    &\left(1 - \frac{\sum_{\bxi \notin \mathcal{K}}  p_t(\bz | \bxi)}{p_t(\bz)}\right) \text{Tr}\left(\text{Cov}\left(\hat{\mu}_{\text{MC}}\right)\right) + \frac{(N-k)}{N} \text{Tr}\left(\text{Cov}\left(\hat{\mu}_{\text{U}}\right)\right).
\end{split}
\end{equation}
\begin{proof}
We start from the expression for the Diagonal of the covariance of an SNIS posterior mean estimator
{\allowdisplaybreaks
\begin{align}
    n \cdot \text{Diag}\left(\text{Cov}\left(\hat{\mu}_{\text{SNIS}}\right)\right) &= \sum_{i=1}^N \frac{p_t(\bxi | \bz)^2}{q_t(\bxi | \bz)}\left(\bxi - \pmean\right)^{\circ2}\\
    &= \sum_{\bxi \in \mathcal{K}} \frac{p_t(\bxi | \bz)^2}{q_t(\bxi | \bz)}\left(\bxi - \pmean\right)^{\circ2} + \sum_{\bxi \notin \mathcal{K}} \frac{p_t(\bxi | \bz)^2}{q_t(\bxi | \bz)}\left(\bxi - \pmean\right)^{\circ2}. \\
\end{align}}
For the KNN based proposal described in \cref{eq:proposal}, we can write
{\allowdisplaybreaks
\begin{align}
    \begin{split}
        n \cdot \text{Diag}\left(\text{Cov}\left(\hat{\mu}_{\text{KNN}}\right)\right) &= \sum_{\bxi \in \mathcal{K}} \frac{p_t(\bxi | \bz)^2}{q_t(\bxi | \bz)}\left(\bxi - \pmean\right)^{\circ2} \\
        &\quad \quad \quad + \sum_{\bxi \notin \mathcal{K}} \frac{\mathcal{Z}_q}{p_t(\bz | \bxk)}p_t(\bxi | \bz)^2\left(\bxi - \pmean\right)^{\circ2}
    \end{split}\\
    \begin{split}
    &=\sum_{\bxi \in \mathcal{K}} \frac{p_t(\bxi | \bz)^2}{q_t(\bxi | \bz)}\left(\bxi - \pmean\right)^{\circ2} \\
    & \quad \quad \quad + \frac{\mathcal{Z}_q}{p_t(\bz | \bxk)} \sum_{\bxi \notin \mathcal{K}} p_t(\bxi | \bz)^2\left(\bxi - \pmean\right)^{\circ2}
    \end{split}\\
    \begin{split}
        &= \sum_{\bxi \in \mathcal{K}} \frac{p_t(\bxi | \bz)^2}{q_t(\bxi | \bz)}\left(\bxi - \pmean\right)^{\circ2} \\
        & \quad \quad \quad + \frac{\mathcal{Z}_q}{p_t(\bz | \bxk)} \sum_{\bxi \notin \mathcal{K}} p_t(\bxi | \bz)^2\left(\bxi - \pmean\right)^{\circ2} \\
        & \quad \quad \quad \pm \frac{\mathcal{Z}_q}{p_t(\bz | \bxk)} \sum_{\bxi \in \mathcal{K}} p_t(\bxi | \bz)^2\left(\bxi - \pmean\right)^{\circ2}
    \end{split}\\
    \begin{split}
        &= \sum_{\bxi \in \mathcal{K}} \frac{p_t(\bxi | \bz)^2}{q_t(\bxi | \bz)}\left(\bxi - \pmean\right)^{\circ2} \\
        & \quad \quad \quad - \frac{\mathcal{Z}_q}{p_t(\bz | \bxk)} \sum_{\bxi \in \mathcal{K}} p_t(\bxi | \bz)^2\left(\bxi - \pmean\right)^{\circ2} \\
        & \quad \quad \quad + \frac{\mathcal{Z}_q}{p_t(\bz | \bxk)} \sum_{i=1}^N p_t(\bxi | \bz)^2\left(\bxi - \pmean\right)^{\circ2}
    \end{split}\\
    \begin{split}
        &= \sum_{\bxi \in \mathcal{K}} \frac{p_t(\bxi | \bz)^2}{q_t(\bxi | \bz)}\left(\bxi - \pmean\right)^{\circ2} \\
        & \quad \quad \quad - \frac{\sum_{\bxi \in \mathcal{K}}p_t(\bz | \bxi) + (N-k)p_t(\bz | \bxk)}{p_t(\bz | \bxk)} \sum_{\bxi \in \mathcal{K}} p_t(\bxi | \bz)^2\left(\bxi - \pmean\right)^{\circ2} \\
        & \quad \quad \quad + \frac{\sum_{\bxi \in \mathcal{K}}p_t(\bz | \bxi) + (N-k)p_t(\bz | \bxk)}{Np_t(\bz | \bxk)}\cdot \sum_{i=1}^N N \cdot p_t(\bxi | \bz)^2\left(\bxi - \pmean\right)^{\circ2}
    \end{split}\\
    \begin{split}
        &= \sum_{\bxi \in \mathcal{K}} \frac{p_t(\bxi | \bz)^2}{q_t(\bxi | \bz)}\left(\bxi - \pmean\right)^{\circ2} \\
        & \quad \quad \quad - \frac{\sum_{\bxi \in \mathcal{K}}p_t(\bz | \bxi)}{p_t(\bz | \bxk)} \sum_{\bxi \in \mathcal{K}} p_t(\bxi | \bz)^2\left(\bxi - \pmean\right)^{\circ2} \\
        & \quad \quad \quad - \frac{(N-k)p_t(\bz | \bxk)}{p_t(\bz | \bxk)} \sum_{\bxi \in \mathcal{K}} p_t(\bxi | \bz)^2\left(\bxi - \pmean\right)^{\circ2} \\
        & \quad \quad \quad + \frac{\sum_{\bxi \in \mathcal{K}}p_t(\bz | \bxi)}{Np_t(\bz | \bxk)}\cdot \sum_{i=1}^N N \cdot p_t(\bxi | \bz)^2\left(\bxi - \pmean\right)^{\circ2}\\
        & \quad \quad \quad + \frac{(N-k)p_t(\bz | \bxk)}{Np_t(\bz | \bxk)}\cdot \sum_{i=1}^N N \cdot p_t(\bxi | \bz)^2\left(\bxi - \pmean\right)^{\circ2}. \label{eq:stf_sub_ref}
    \end{split}
\end{align}}
We substitute \cref{eq:uniform_trace} into last term of \cref{eq:stf_sub_ref}. Continuing, we have
{\allowdisplaybreaks
\begin{align} 
    \begin{split}
        n \cdot \text{Diag}\left(\text{Cov}\left(\hat{\mu}_{\text{KNN}}\right)\right) &= \sum_{\bxi \in \mathcal{K}} \frac{p_t(\bxi | \bz)^2}{q_t(\bxi | \bz)}\left(\bxi - \pmean\right)^{\circ2} \\
        & \quad \quad \quad - \frac{\sum_{\bxi \in \mathcal{K}}p_t(\bz | \bxi)}{p_t(\bz | \bxk)} \sum_{\bxi \in \mathcal{K}} p_t(\bxi | \bz)^2\left(\bxi - \pmean\right)^{\circ2} \\
        & \quad \quad \quad - \frac{(N-k)p_t(\bz | \bxk)}{p_t(\bz | \bxk)} \sum_{\bxi \in \mathcal{K}} p_t(\bxi | \bz)^2\left(\bxi - \pmean\right)^{\circ2} \\
        & \quad \quad \quad + \frac{\sum_{\bxi \in \mathcal{K}}p_t(\bz | \bxi)}{Np_t(\bz | \bxk)}\cdot \sum_{i=1}^N N \cdot p_t(\bxi | \bz)^2\left(\bxi - \pmean\right)^{\circ2}\\
        & \quad \quad \quad + \frac{(N-k)}{N}\cdot n \cdot \text{Diag}\left(\text{Cov}\left(\hat{\mu}_{\text{U}}\right) \right)
    \end{split}\\
    \begin{split}
        &= \sum_{\bxi \in \mathcal{K}} \frac{p_t(\bxi | \bz)^2}{q_t(\bxi | \bz)}\left(\bxi - \pmean\right)^{\circ2} \\
        & \quad \quad \quad - \frac{(N-k)p_t(\bz | \bxk)}{p_t(\bz | \bxk)} \sum_{\bxi \in \mathcal{K}} p_t(\bxi | \bz)^2\left(\bxi - \pmean\right)^{\circ2} \\
        & \quad \quad \quad + \frac{\sum_{\bxi \in \mathcal{K}}p_t(\bz | \bxi)}{p_t(\bz | \bxk)}\cdot \sum_{\bxi \notin \mathcal{K}}p_t(\bxi | \bz)^2\left(\bxi - \pmean\right)^{\circ2}\\
        & \quad \quad \quad + \frac{(N-k)}{N}\cdot n \cdot \text{Diag}\left(\text{Cov}\left(\hat{\mu}_{\text{U}}\right) \right)
    \end{split}\\ % Maybe do the same bound on term 2???
    \begin{split}
        &= \sum_{\bxi \in \mathcal{K}} \frac{p_t(\bxi | \bz)^2}{q_t(\bxi | \bz)}\left(\bxi - \pmean\right)^{\circ2} \\
        & \quad \quad \quad - \frac{(N-k)p_t(\bz | \bxk)}{p_t(\bz)} \sum_{\bxi \in \mathcal{K}} \frac{p_t(\bz | \bxi)}{p_t(\bz | \bxk)}p_t(\bxi | \bz)\left(\bxi - \pmean\right)^{\circ2} \\
        & \quad \quad \quad + \frac{\sum_{\bxi \in \mathcal{K}}p_t(\bz | \bxi)}{p_t(\bz)}\cdot \sum_{\bxi \notin \mathcal{K}}\frac{p_t(\bz | \bxi)}{p_t(\bz | \bxk)}p_t(\bxi | \bz)\left(\bxi - \pmean\right)^{\circ2}\\
        & \quad \quad \quad + \frac{(N-k)}{N}\cdot n \cdot \text{Diag}\left(\text{Cov}\left(\hat{\mu}_{\text{U}}\right) \right).
    \end{split}\label{eq:substitute_reference}
\end{align}}
We note that for $\bxi \notin \mathcal{K}$, $\frac{p_t(\bz | \bxi)}{p_t(\bz | \bxk)} \leq 1$. Similarly for $\bxi \in \mathcal{K}$, $\frac{p_t(\bz | \bxi)}{p_t(\bz | \bxk)} \geq 1$ Swapping these fractions with $1$ in the second and third term upper bounds each dimension of the trace of covariance. We substitute $\frac{p_t(\bz | \bxi)}{p_t(\bz | \bxk)} = 1$ in the second and third terms of \cref{eq:substitute_reference}, and swap the equality with a element-wise inequality
{\allowdisplaybreaks
\begin{align} 
    \begin{split}
        n \cdot \text{Diag}\left(\text{Cov}\left(\hat{\mu}_{\text{KNN}}\right)\right) &\leq \sum_{\bxi \in \mathcal{K}} \frac{p_t(\bxi | \bz)^2}{q_t(\bxi | \bz)}\left(\bxi - \pmean\right)^{\circ2} \\
        & \quad \quad \quad - \frac{(N-k)p_t(\bz | \bxk)}{p_t(\bz)} \sum_{\bxi \in \mathcal{K}} p_t(\bxi | \bz)\left(\bxi - \pmean\right)^{\circ2} \\
        & \quad \quad \quad + \frac{\sum_{\bxi \in \mathcal{K}}p_t(\bz | \bxi)}{p_t(\bz)}\cdot \sum_{\bxi \notin \mathcal{K}}p_t(\bxi | \bz)\left(\bxi - \pmean\right)^{\circ2}\\
        & \quad \quad \quad + \frac{(N-k)}{N}\cdot n \cdot \text{Diag}\left(\text{Cov}\left(\hat{\mu}_{\text{U}}\right) \right)
    \end{split}\\
    \begin{split}
        &\leq \sum_{\bxi \in \mathcal{K}} \frac{p_t(\bxi | \bz)^2}{q_t(\bxi | \bz)}\left(\bxi - \pmean\right)^{\circ2} \\
        & \quad \quad \quad - \frac{(N-k)p_t(\bz | \bxk)}{p_t(\bz)} \sum_{\bxi \in \mathcal{K}} p_t(\bxi | \bz)\left(\bxi - \pmean\right)^{\circ2} \\
        & \quad \quad \quad + \frac{\mathcal{Z}_q - (N-k)p_t(\bz | \bxk)}{p_t(\bz)}\cdot \sum_{\bxi \notin \mathcal{K}}p_t(\bxi | \bz)\left(\bxi - \pmean\right)^{\circ2}\\
        & \quad \quad \quad + \frac{(N-k)}{N}\cdot n \cdot \text{Diag}\left(\text{Cov}\left(\hat{\mu}_{\text{U}}\right) \right)
    \end{split}\\
    \begin{split}
        &\leq \frac{\mathcal{Z}_q}{p_t(\bz)}\sum_{\bxi \in \mathcal{K}} \frac{p_t(\bz | \bxi)}{p_t(\bz | \bxi)} p_t(\bxi | \bz)\left(\bxi - \pmean\right)^{\circ2} \\
        & \quad \quad \quad - \frac{(N-k)p_t(\bz | \bxk)}{p_t(\bz)} \sum_{\bxi \in \mathcal{K}} p_t(\bxi | \bz)\left(\bxi - \pmean\right)^{\circ2} \\
        & \quad \quad \quad + \frac{\mathcal{Z}_q - (N-k)p_t(\bz | \bxk)}{p_t(\bz)}\cdot \sum_{\bxi \notin \mathcal{K}}p_t(\bxi | \bz)\left(\bxi - \pmean\right)^{\circ2}\\
        & \quad \quad \quad + \frac{(N-k)}{N}\cdot n \cdot \text{Diag}\left(\text{Cov}\left(\hat{\mu}_{\text{U}}\right) \right)
    \end{split}\\ % HERE
        \begin{split}
        &\leq \frac{\mathcal{Z}_q}{p_t(\bz)}\sum_{i=1}^N p_t(\bxi | \bz)\left(\bxi - \pmean\right)^{\circ2} \\
        & \quad \quad \quad - \frac{(N-k)p_t(\bz | \bxk)}{p_t(\bz)} \sum_{i=1}^N p_t(\bxi | \bz)\left(\bxi - \pmean\right)^{\circ2} \\
        & \quad \quad \quad + \frac{(N-k)}{N}\cdot n \cdot \text{Diag}\left(\text{Cov}\left(\hat{\mu}_{\text{U}}\right) \right).
    \end{split} \label{eq:mc_trace_sub} % HERE
\end{align}}

We substitute $n \cdot \text{Diag}\left(\text{Cov}\left(\hat{\mu}_{\text{MC}}\right)\right) = \sum_{i=1}^N p_t(\bxi | \bz)\left(\bxi - \pmean\right)^{\circ2}$ into the second and third terms of \cref{eq:mc_trace_sub}

{\allowdisplaybreaks
\begin{align}
    \begin{split}
        n \cdot \text{Diag}\left(\text{Cov}\left(\hat{\mu}_{\text{KNN}}\right)\right) &\leq \frac{\mathcal{Z}_q}{p_t(\bz)}\cdot n \cdot \text{Diag}\left(\text{Cov}\left(\hat{\mu}_{\text{MC}}\right)\right) \\
        & \quad \quad \quad - \frac{(N-k)p_t(\bz | \bxk)}{p_t(\bz)} \cdot n \cdot \text{Diag}\left(\text{Cov}\left(\hat{\mu}_{\text{MC}}\right)\right)\\
        & \quad \quad \quad + \frac{(N-k)}{N}\cdot n \cdot \text{Diag}\left(\text{Cov}\left(\hat{\mu}_{\text{U}}\right) \right)
    \end{split}\\
    \begin{split}
          &\leq \frac{\mathcal{Z}_q - (N-k)p_t(\bz | \bxk)}{p_t(\bz)} \cdot n \cdot \text{Diag}\left(\text{Cov}\left(\hat{\mu}_{\text{MC}}\right)\right)\\
        & \quad \quad \quad + \frac{(N-k)}{N}\cdot n \cdot \text{Diag}\left(\text{Cov}\left(\hat{\mu}_{\text{U}}\right) \right)
    \end{split}\\ % HERE
       \begin{split}
        &\leq \frac{\sum_{\bxi \in \mathcal{K}} p_t(\bz | \bxi)}{p_t(\bz)} \cdot n \cdot \text{Diag}\left(\text{Cov}\left(\hat{\mu}_{\text{MC}}\right)\right)\\
        & \quad \quad \quad + \frac{(N-k)}{N}\cdot n \cdot \text{Diag}\left(\text{Cov}\left(\hat{\mu}_{\text{U}}\right) \right)
    \end{split}\\ % HERE
   \begin{split}
        &\leq \left(1 - \frac{\sum_{\bxi \notin \mathcal{K}} p_t(\bz | \bxi)}{p_t(\bz)}\right) \cdot n \cdot \text{Diag}\left(\text{Cov}\left(\hat{\mu}_{\text{MC}}\right)\right)\\
        & \quad \quad \quad + \frac{(N-k)}{N}\cdot n \cdot \text{Diag}\left(\text{Cov}\left(\hat{\mu}_{\text{U}}\right). \right)
    \end{split}
\end{align}
}
Dividing both sides by $n$ yields 
\begin{equation}
    \begin{split}
        &\text{Diag}\left(\text{Cov}\left(\hat{\mu}_{\text{KNN}}\right)\right) \leq \\
        &\left(1 - \frac{\sum_{\bxi \notin \mathcal{K}} p_t(\bz | \bxi)}{p_t(\bz)}\right)\text{Diag}\left(\text{Cov}\left(\hat{\mu}_{\text{MC}}\right)\right) + \frac{(N-k)}{N}\text{Diag}\left(\text{Cov}\left(\hat{\mu}_{\text{U}}\right)\right).
    \end{split} \label{eq:diag_stf_vs_knn}
\end{equation}

Let $\text{Diag}\left(\text{Cov}\left(\cdot\right) \right)^j$ denote the $j^{th}$ component of the diagonal vector. Since the inequaltiy in \cref{eq:diag_stf_vs_knn} is expressed element wise, $\forall j \in [0,\ldots, d]$ we have

{\allowdisplaybreaks
\begin{align}
    \begin{split}
        &\text{Diag}\left(\text{Cov}\left(\hat{\mu}_{\text{KNN}}\right)\right)^j \leq \\
        &\left(1 - \frac{\sum_{\bxi \notin \mathcal{K}} p_t(\bz | \bxi)}{p_t(\bz)}\right)\text{Diag}\left(\text{Cov}\left(\hat{\mu}_{\text{MC}}\right)\right)^j + \frac{(N-k)}{N}\text{Diag}\left(\text{Cov}\left(\hat{\mu}_{\text{U}}\right)\right)^j.
    \end{split} \label{eq:jth_component_knn_v_stf}
\end{align}
}
Summing \cref{eq:jth_component_knn_v_stf} from $j=1$ to $j=d$, we get
{\allowdisplaybreaks
\begin{align}
    \begin{split}
        \sum_{j=1}^d\text{Diag} \left(\text{Cov}\left(\hat{\mu}_{\text{KNN}}\right)\right)^j \leq \sum_{j=1}^d\Bigg(\left(1 - \frac{\sum_{\bxi \notin \mathcal{K}} p_t(\bz | \bxi)}{p_t(\bz)}\right)\text{Diag}\left(\text{Cov}\left(\hat{\mu}_{\text{MC}}\right)\right)^j + \frac{(N-k)}{N}\text{Diag}\left(\text{Cov}\left(\hat{\mu}_{\text{U}}\right)\right)^j\Bigg)
    \end{split}\\
    \begin{split}
        \text{Tr}\left(\text{Cov}\left(\hat{\mu}_{\text{KNN}}\right)\right) \leq \left(1 - \frac{\sum_{\bxi \notin \mathcal{K}} p_t(\bz | \bxi)}{p_t(\bz)}\right)\sum_{j=1}^d \text{Diag}\left(\text{Cov}\left(\hat{\mu}_{\text{MC}}\right)\right)^j + \frac{(N-k)}{N}\sum_{j=1}^d\text{Diag}\left(\text{Cov}\left(\hat{\mu}_{\text{U}}\right)\right)^j
    \end{split}\\
    \begin{split}
        \text{Tr}\left(\text{Cov}\left(\hat{\mu}_{\text{KNN}}\right)\right) \leq \left(1 - \frac{\sum_{\bxi \notin \mathcal{K}} p_t(\bz | \bxi)}{p_t(\bz)}\right)\text{Tr}\left(\text{Cov}\left(\hat{\mu}_{\text{MC}}\right)\right) + \frac{(N-k)}{N}\text{Tr}\left(\text{Cov}\left(\hat{\mu}_{\text{U}}\right)\right),
    \end{split}
\end{align}
}
concluding the proof.
\end{proof}
\section{Additional Results}
{\subsection{Score Estimator Hyperparameter Ablation} \label{ap:estimator_ablation}}

We include additional results from the investigation of estimator errors from \cref{sec:empirical_perf}. \Cref{fig:additional_quantitative_n64,fig:additional_quantitative_n256} show the performance of the posterior Monte Carlo, STF, and KNN estimators for varying $n$ and $k$. In both cases, we see that the KNN estimator improves with $k$ with low bias for all estimators, and lower MSE than STF. Increasing the sample size $n$ reduces MSE substantially for our method, but not STF due to the prevalent bias of their estimator. We see that for small $k$ and $n$, the KNN estimator exhibits similar variance to STF. We hypothesize that becuase STF deterministically includes only one image from the posterior, the weights of this image dominates that of the other images drawn from the data distribution, in turn leading to lower sample variance.

\begin{figure}
    \centering
    \includegraphics[width=0.8\textwidth]{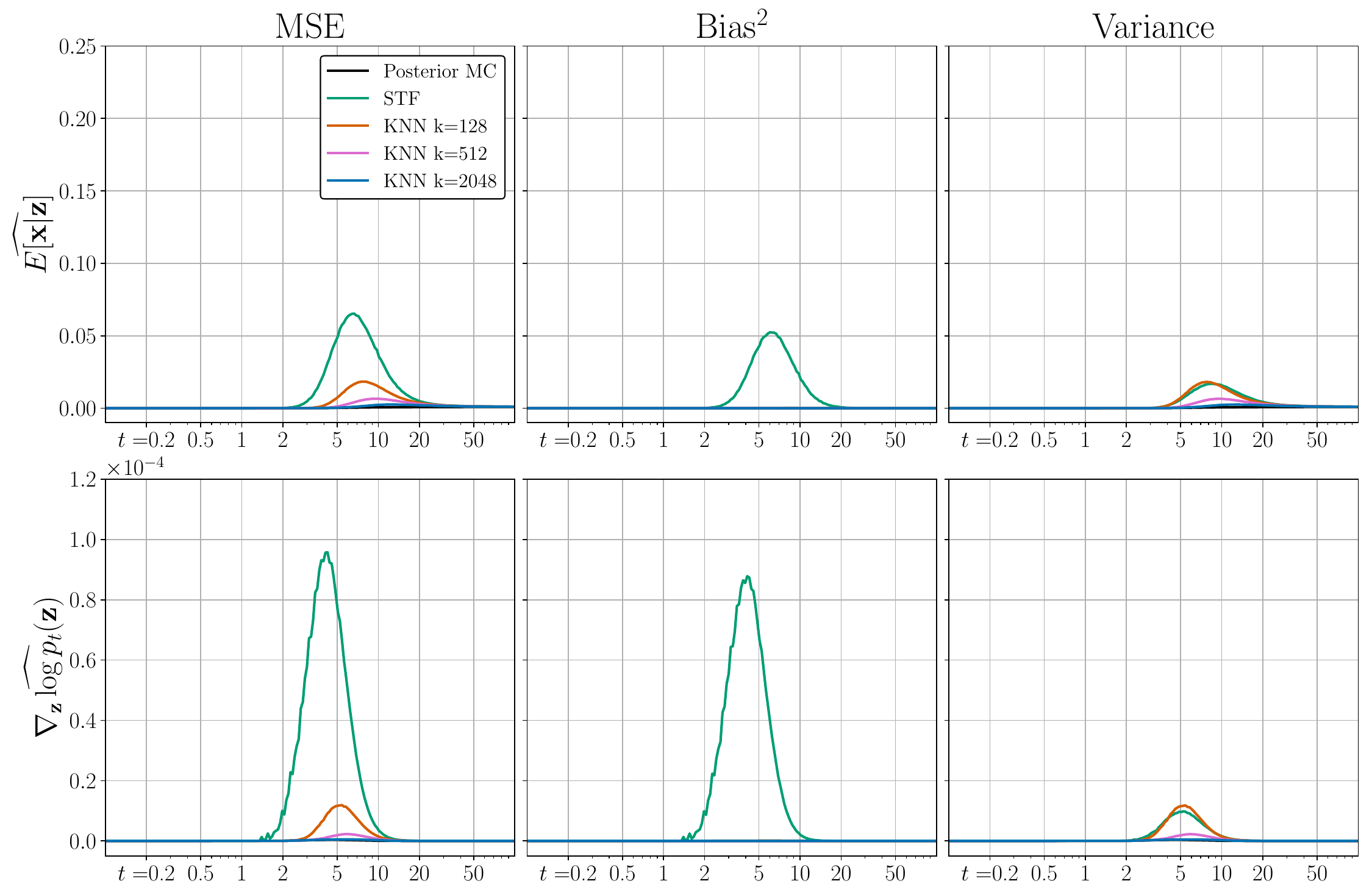}
    \caption{Estimator performance for $n=256$. KNN estimator performance generally improves with $k$.} 
    \label{fig:additional_quantitative_n256}
\end{figure}

\begin{figure}
    \centering
    \includegraphics[width=0.8\textwidth]{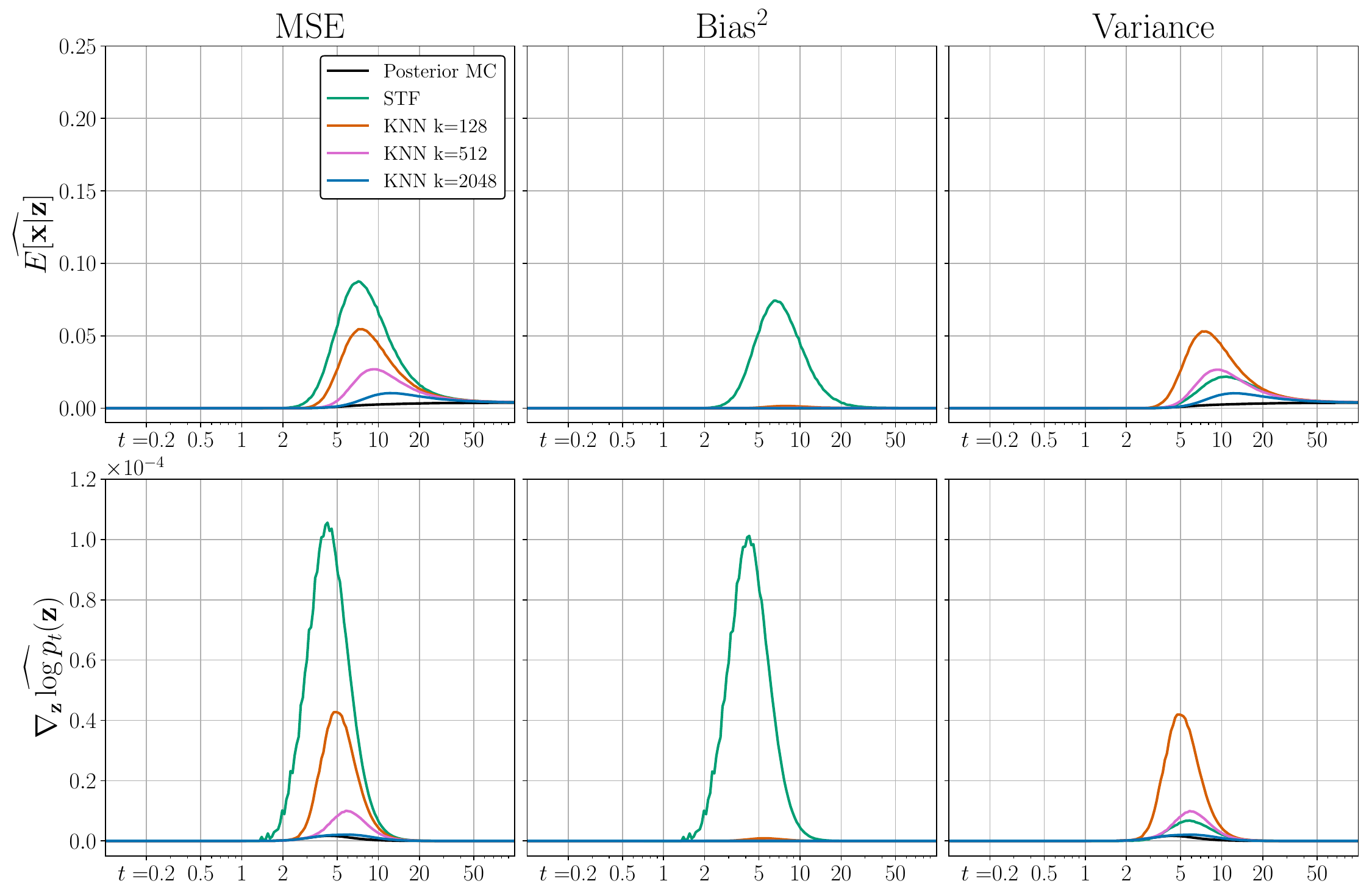}
    \caption{Estimator performance for $n=64$. KNN estimator performance generally improves with $k$.}
    \label{fig:additional_quantitative_n64}
\end{figure}

\subsection{Additional Denoiser Images}

\Cref{fig:additional_qualitative} shows additional KNN posterior mean estimates for a variety of source images and noise levels.

\begin{figure*}[ht!]
    \centering
    \begin{minipage}[b]{.5\textwidth}
        \centering
        \includegraphics[width=\textwidth]{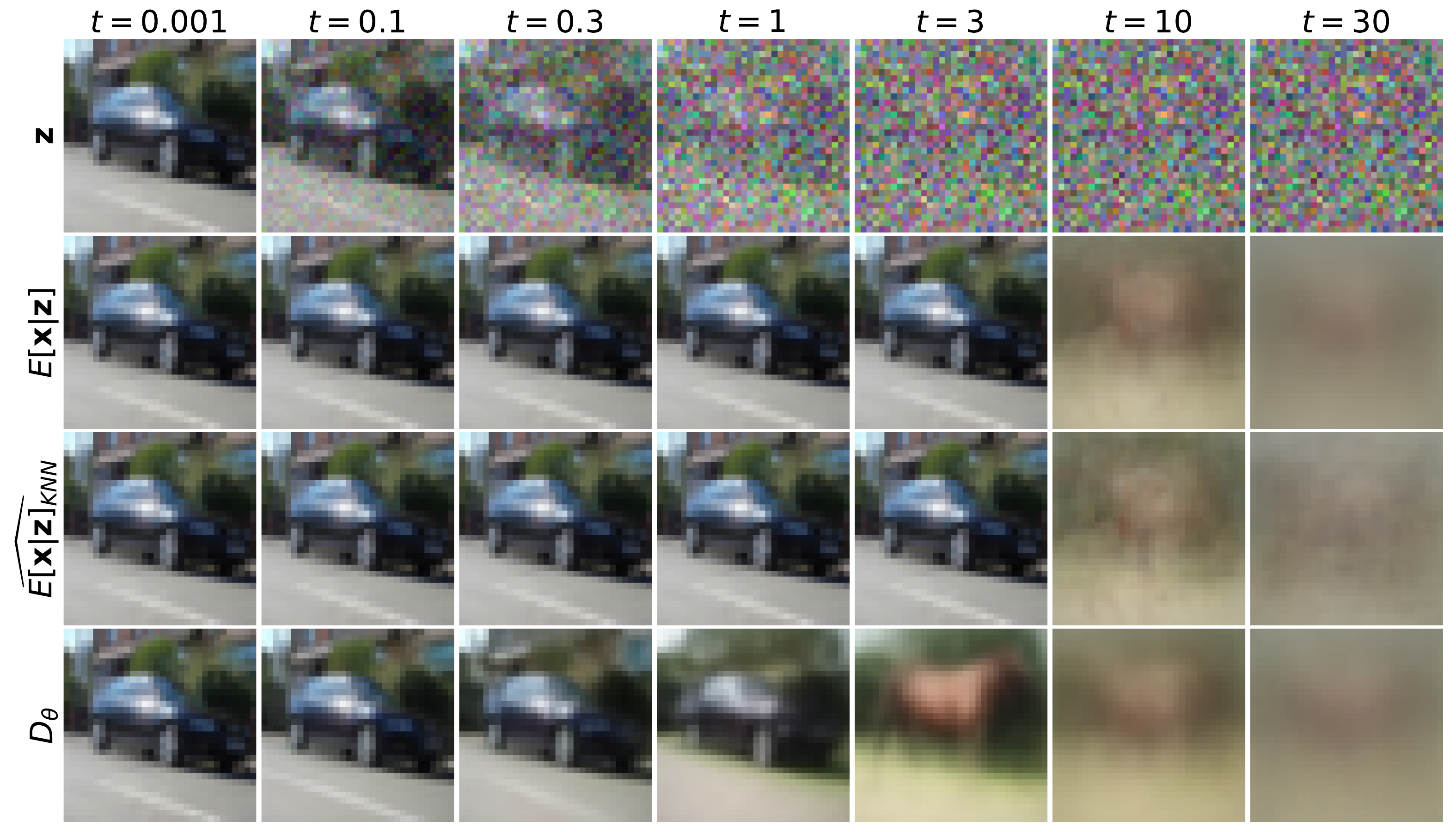}
    \end{minipage}\hfill
    \begin{minipage}[b]{.5\textwidth}
        \centering
        \includegraphics[width=\textwidth]{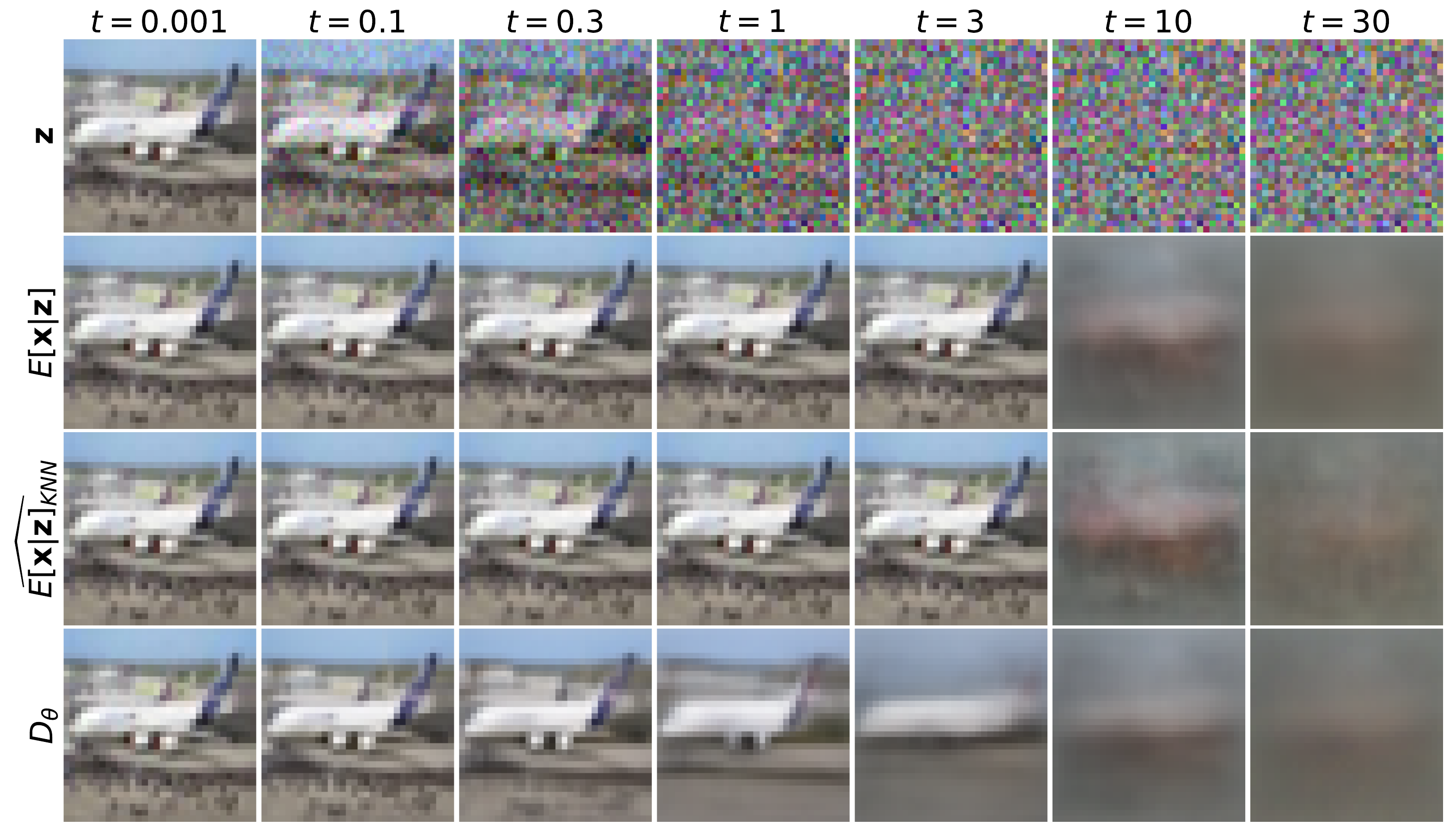}
    \end{minipage}
    \begin{minipage}[b]{.5\textwidth}
        \centering
        \includegraphics[width=\textwidth]{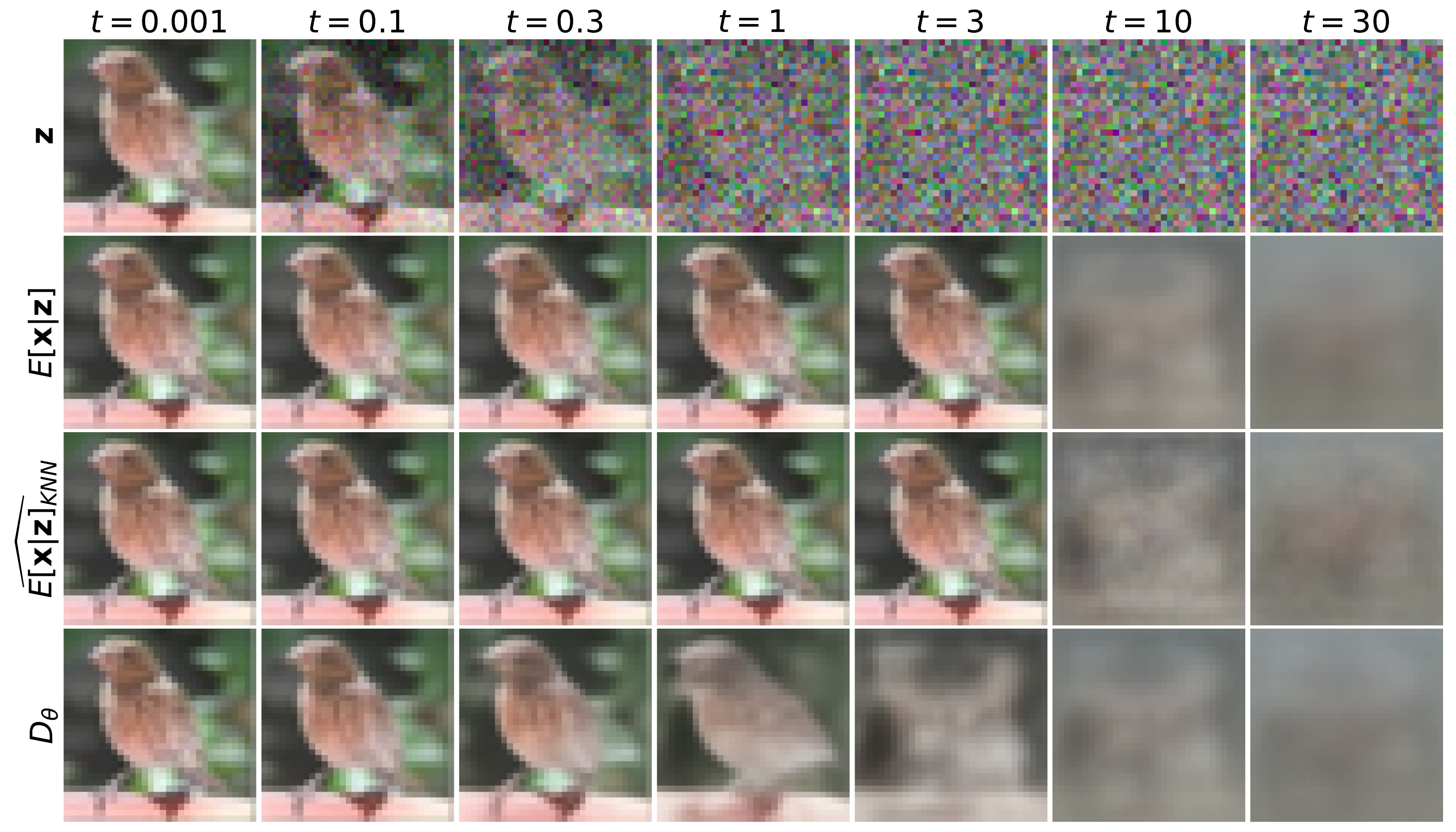}
    \end{minipage}\hfill
    \begin{minipage}[b]{.5\textwidth}
        \centering
        \includegraphics[width=\textwidth]{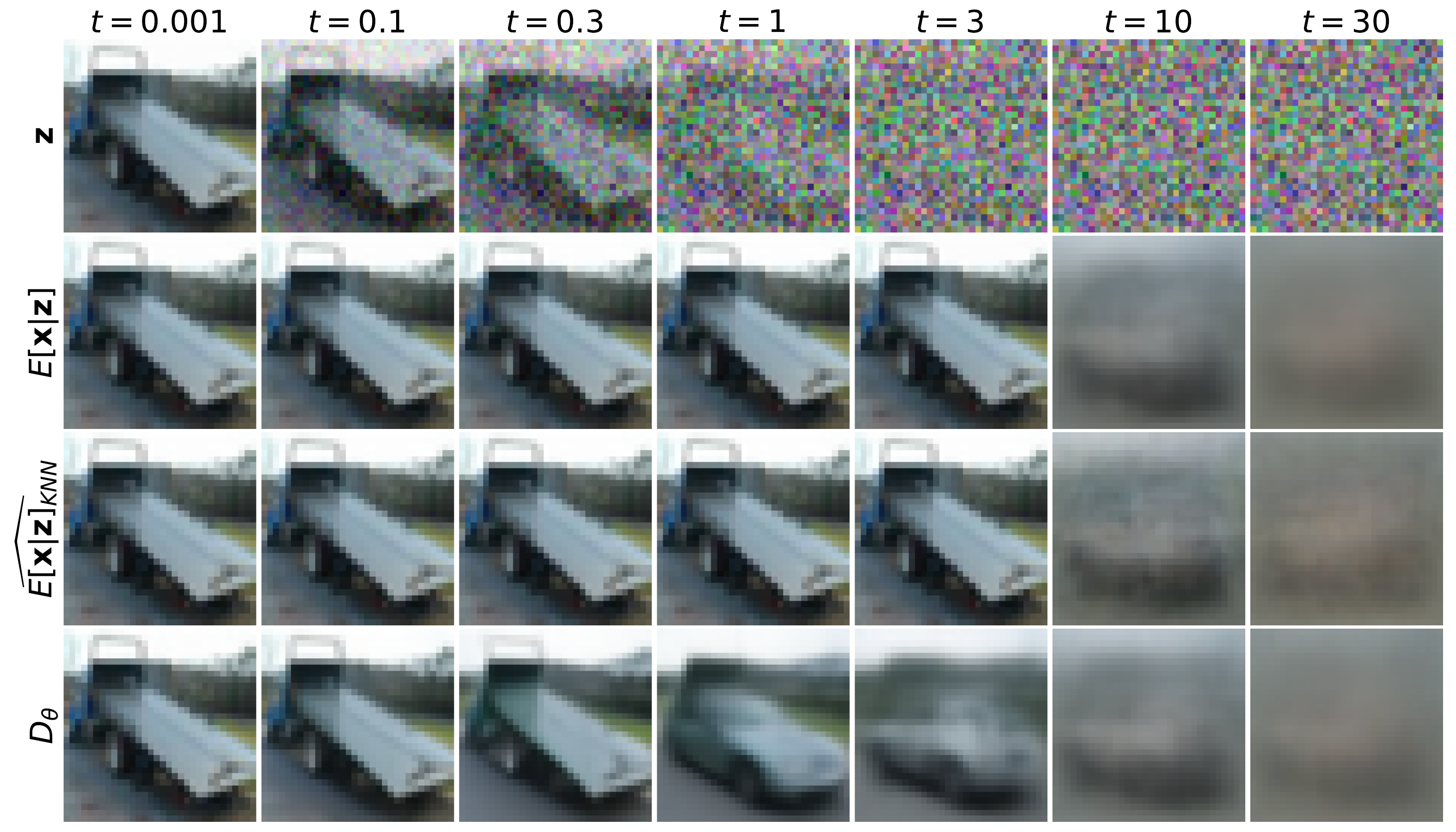}
    \end{minipage}
    \caption{Additional qualitative comparisons of our estimator versus the ground truth posterior mean and a trained EDM denoiser.}
    \label{fig:additional_qualitative}
\end{figure*}

\subsection{CIFAR-10 Consistency Model Samples}

\Cref{fig:cifar_samples} shows examples from our trained consistency models for the models reported in \cref{tab:cifar_fid}. The samples are drawn with consistent seeds for each of the three models.

\begin{figure}[ht!]
    \centering
    \begin{subfigure}% Adjust the width as needed
        \centering
        \includegraphics[width=\textwidth]{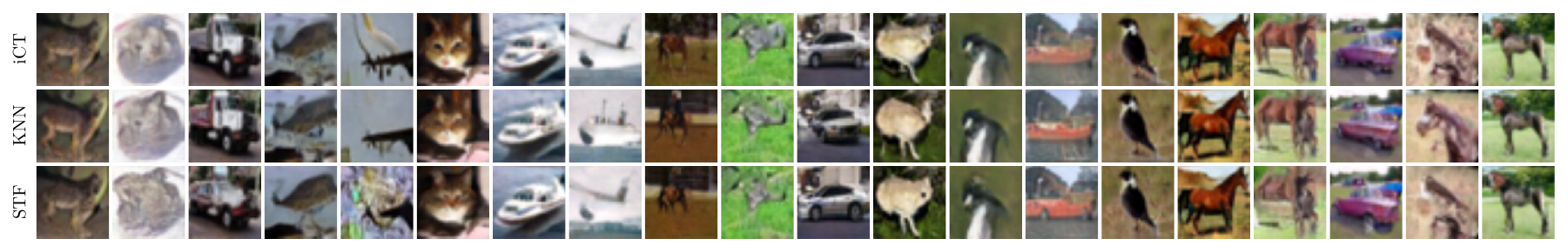} % Adjust the width of the first image
    \end{subfigure}
    
    \begin{subfigure}% Adjust the width as needed
        \centering
        \includegraphics[width=\textwidth]{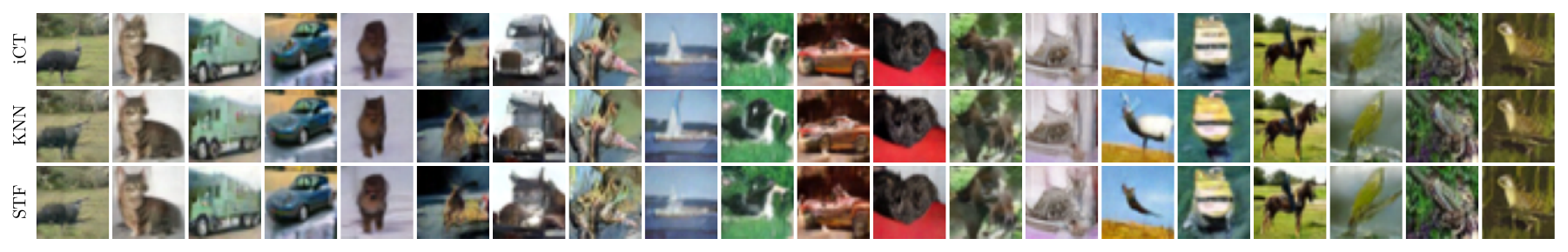} % Adjust the width of the second image
    \end{subfigure}

    \begin{subfigure}% Adjust the width as needed
        \centering
        \includegraphics[width=\textwidth]{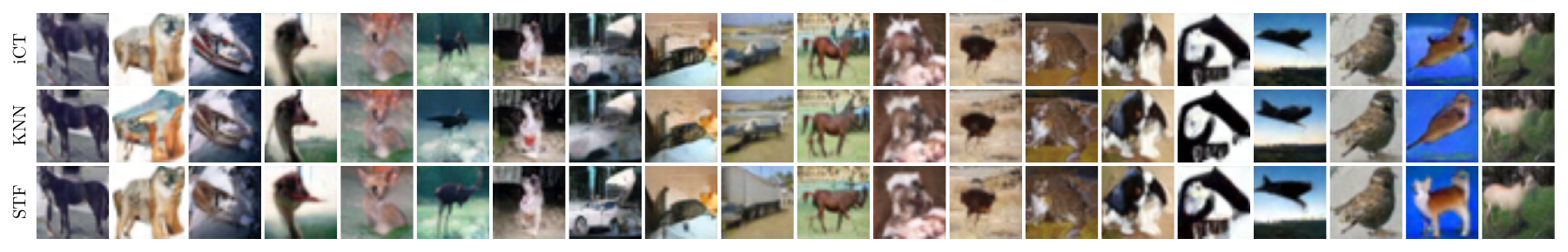} % Adjust the width of the second image
    \end{subfigure}

    \begin{subfigure}% Adjust the width as needed
        \centering
        \includegraphics[width=\textwidth]{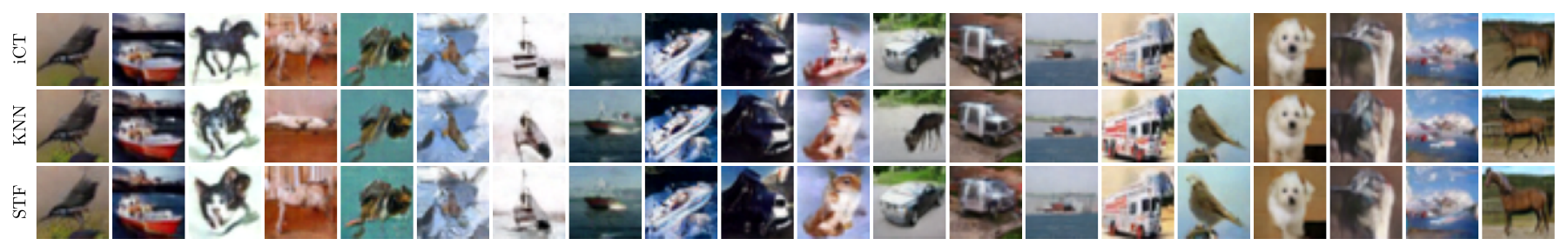} % Adjust the width of the second image
    \end{subfigure}

    \begin{subfigure}% Adjust the width as needed
        \centering
        \includegraphics[width=\textwidth]{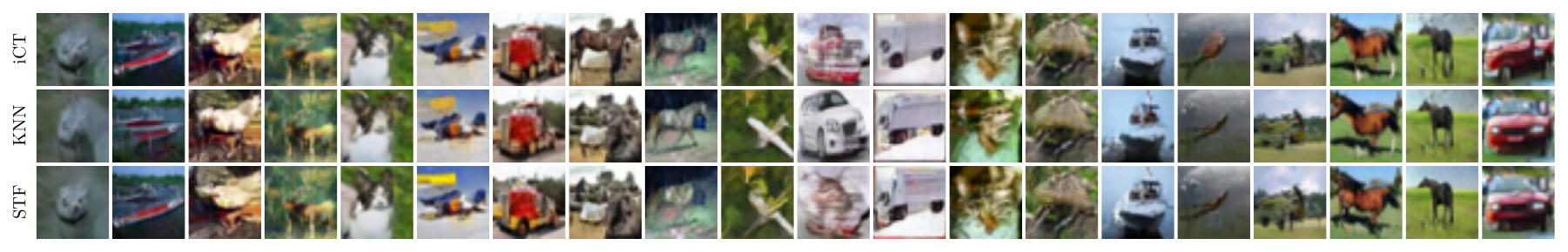} % Adjust the width of the second image
    \end{subfigure}

    \begin{subfigure}% Adjust the width as needed
        \centering
        \includegraphics[width=\textwidth]{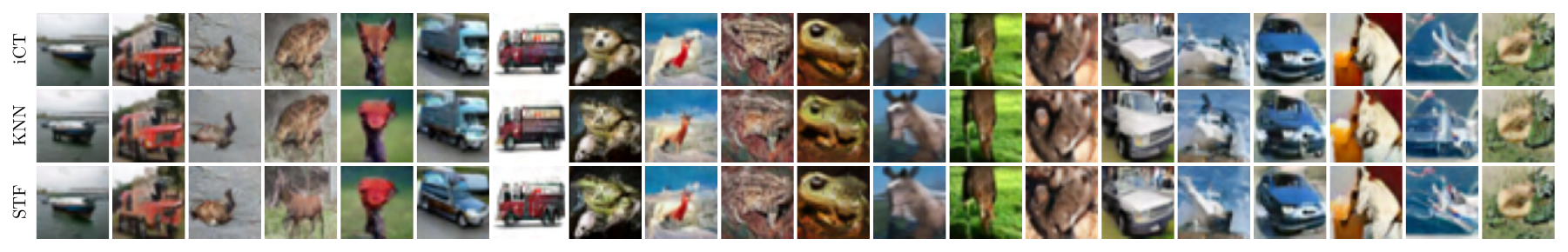} % Adjust the width of the second image
    \end{subfigure}

    % \begin{subfigure}% Adjust the width as needed
    %     \centering
    %     \includegraphics[width=\textwidth]{img/appendix/composite_7.pdf} % Adjust the width of the second image
    % \end{subfigure}
    \caption{CIFAR-10 samples from trained consistency models with varying score estimators.}
    \label{fig:cifar_samples}
\end{figure}

\subsection{Diffusion Models}

Although we demonstrate our method's effectiveness for consistency model training in \cref{sec:consistency}, our method is not specifically tailored to consistency models, and may be used during model training whenever a empirical score estimator is required. In particular, our estimator can be directly applied to train diffusion models.

To evaluate the effect of our estimator on diffusion model training, we train and evaluate unconditional EDM \cite{karras2022elucidating} models on CIFAR-10 \cite{krizhevsky2009learning} using varying score estimators. Specifically, we train models using the standard single-sample Monte Carlo estimator, the STF estimator \cite{xu2023stable}, and our nearest neighbour estimator. We retrain EDM and STF from their official repositories, utilizing the DDPM++ architecture for all models. We report the best FID and IS on 50,000 samples across all model checkpoints in \cref{tab:diffusion_results}

\begin{table}[h!]
    \centering
    \caption{Unconditional EDM CIFAR-10 sample quality with varying score estimators.}
    \begin{tabular}{l r  r }
        \hline
         \textbf{Method}&  \textbf{FID}$\downarrow$&  \textbf{IS}$\uparrow$ \\
         \hline
         EDM \cite{karras2022elucidating} & 2.01 & 9.76  \\
         STF \cite{xu2023stable}& \textbf{1.94} & 9.81 \\
         EDM + KNN (ours) & 1.96 & \textbf{9.87}\\
         \hline
    \end{tabular}
    \label{tab:diffusion_results}
\end{table}

From \cref{tab:diffusion_results}, we see that both our method and STF improve the quality of generated samples over EDM. Compared to STF, our method has similar FID and better Inception Score.

\subsection{Score Estimation on Additional Datasets}

We apply our approach to two additional datasets to evaluate the efficacy of our method on larger and higher dimensionality data. \Cref{fig:celeba} shows the performance of our score estimator on CelebA 64x64 \cite{liu2015faceattributes}, while \cref{fig:imagenet} demonstrates the score estimation error of out method on Imagenet 64x64 \cite{deng2009imagenet}.

\begin{figure}[ht!]
    \centering
    \includegraphics[width=0.8 \textwidth]{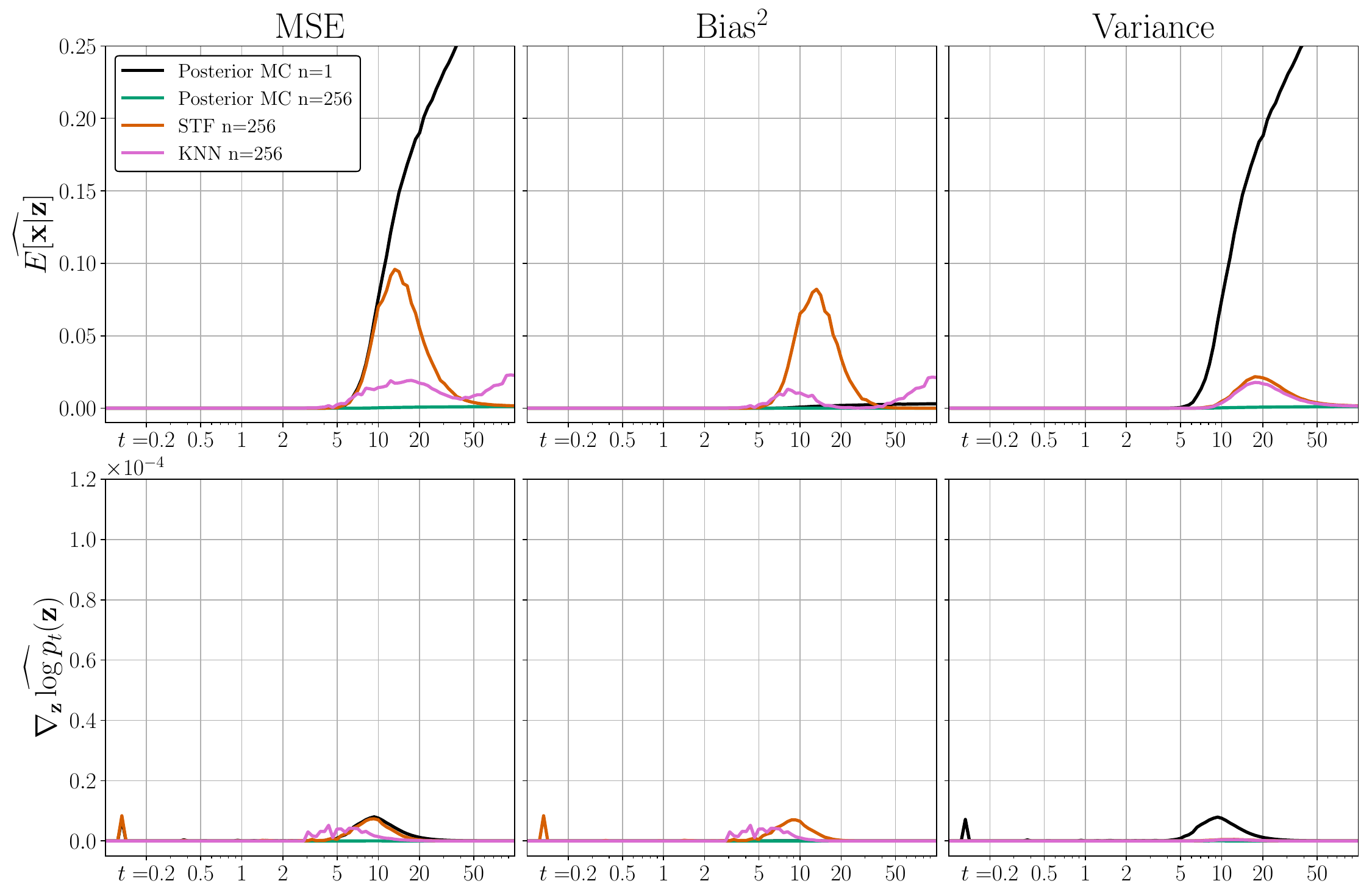}
    \caption{Comparison of various score estimators on unconditional CelebA 64x64}
    \label{fig:celeba}
\end{figure}

\begin{figure} [ht!]
    \centering
    \includegraphics[width= 0.8 \textwidth]{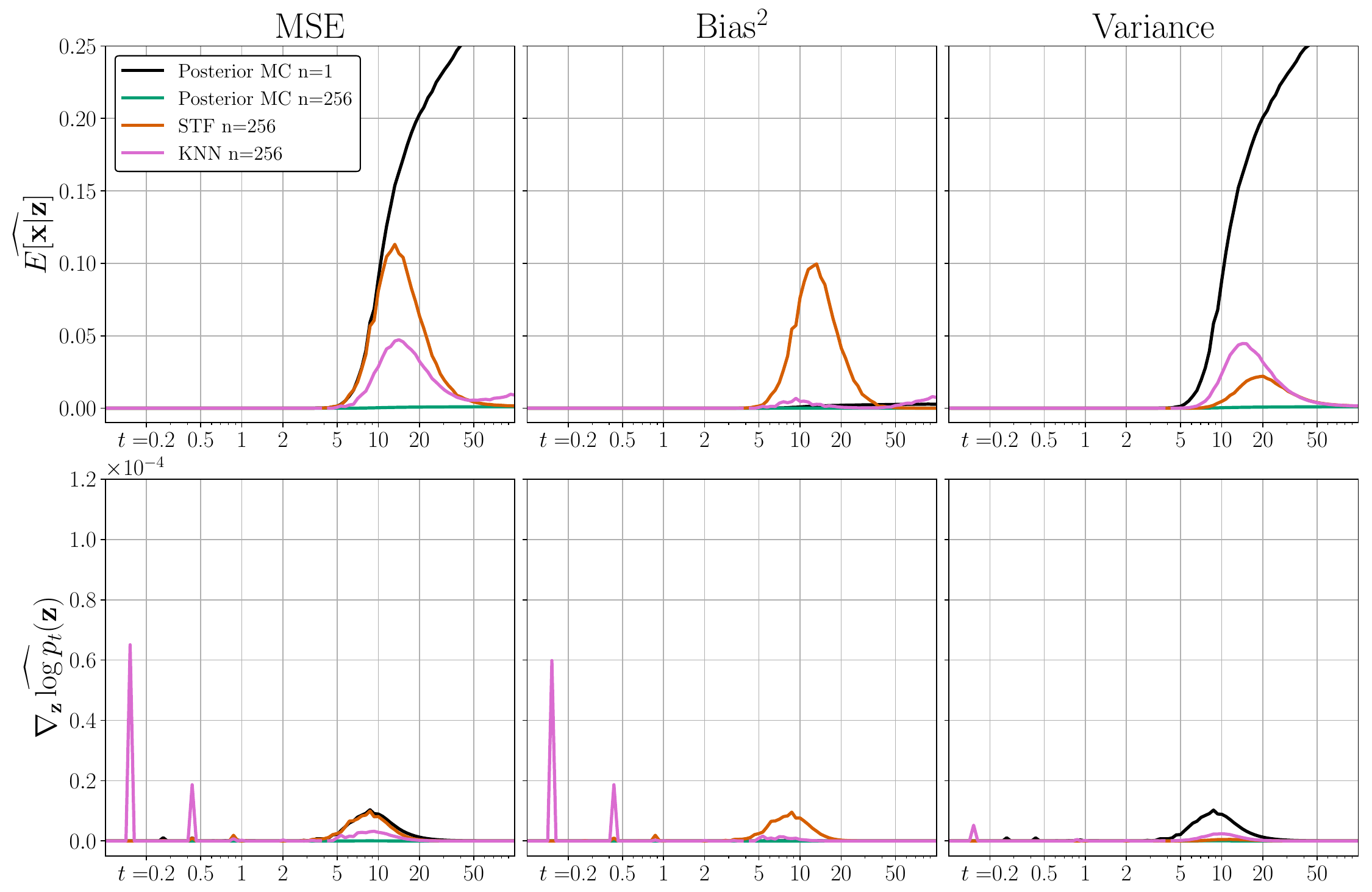}
    \caption{Comparison of various score estimators on unconditional Imagenet 64x64}
    \label{fig:imagenet}
\end{figure}

Applying our method to higher dimensional and larger datasets presents two challenges. Firstly, to perform fast nearest neighbours search, we generate an index of the entire dataset which is stored in memory. Larger images and larger datasets therefore present memory concerns. Secondly, larger datasets increase the runtime of nearest neighbour search, which especially apparent on large datasets such as Imagenet. 

To address these challenges, we apply two strategies. To reduce the memory overhead of our index, we quantize the data to 8 bit precision, reducing memory overhead by a factor of four. This has a minimal affect on our estimator, as pixel intensities are generally represented as 8 bit values. To improve the runtime of nearest neighbour search, we utilize an inverted file index (IVF) which partitions the data into $c$ clusters and performs nearest neighbour search over only the $p$ closest partitions. Both $c$ and $p$ can be tuned to trade off runtime versus accuracy. For CelebA 64x64 we utilized $c=1788$ and $p=25$ while for Imagenet, we used $c=4527$ and $p=100$. 

In our work, we assume that the nearest neighbour search returns the $k$ most probable elements of the posterior distribution. However, by switching to an approximate nearest neighbour search, this may no longer be true. We evaluate the impact of using approximate nearest neighbour search on estimator performance in \cref{fig:approximate_compare}.

\begin{figure} [h!]
    \centering
    \includegraphics[width=0.8\textwidth]{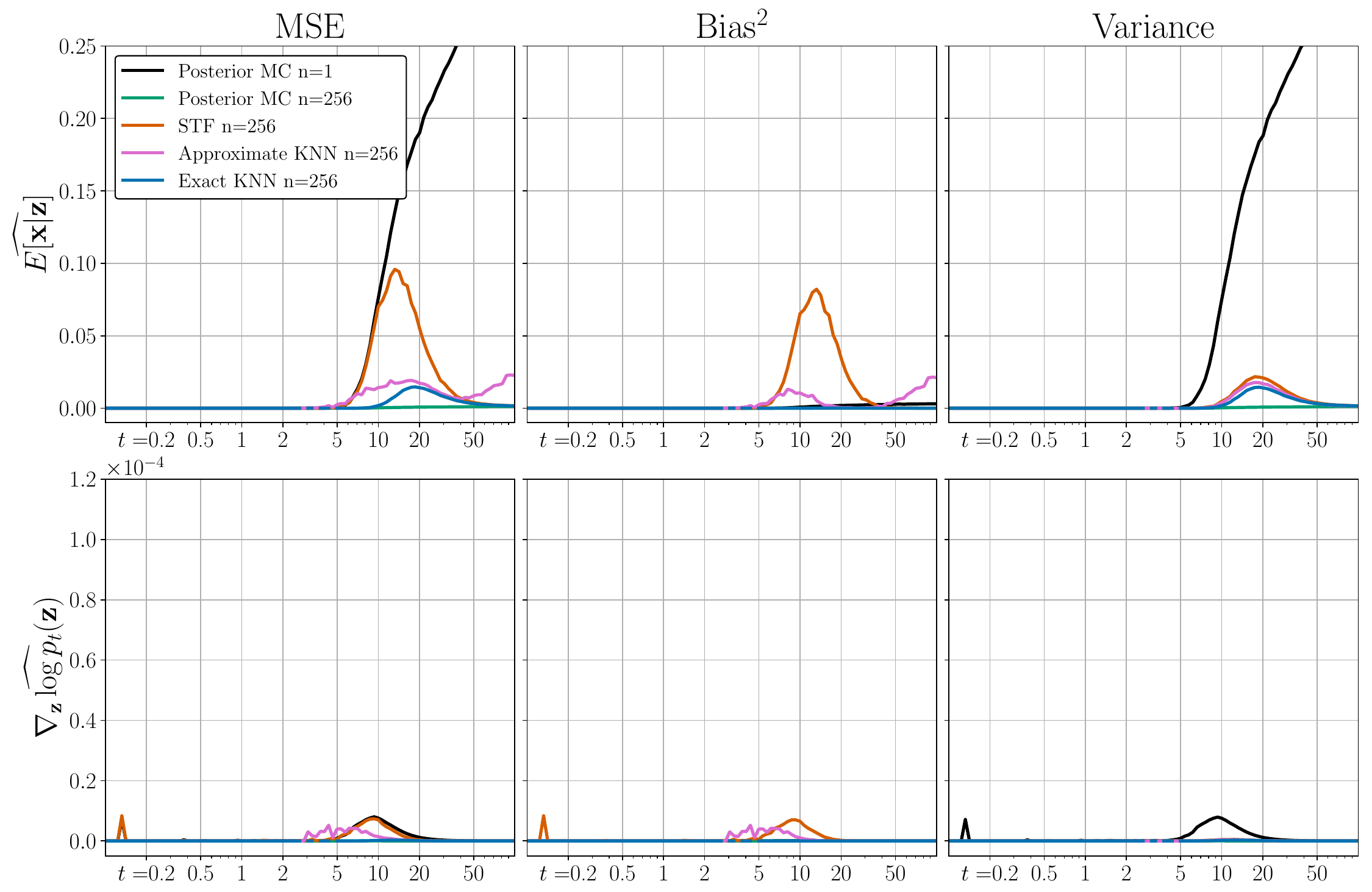}
    \caption{Comparison of approximate and exact nearest neighbour score estimators on CelebA 64x64}
    \label{fig:approximate_compare}
\end{figure}

From \cref{fig:approximate_compare}, we can see that approximate nearest neighbour introduces some additional bias to the posterior mean estimation, however the bias is still less that STF for the majority of noise levels. We expect bias could be reduced by increasing the number of samples $n$ or the number of search partitions $p$.
\section{Experimental Details}
\subsection{Consistency Model Training Configurations}

\begin{table}[ht!]
    \centering
    \caption{Hyperparameter choices for CIFAR-10 consistency model training}
    \begin{tabular}{l| c c c}
    \hline
         Hyperparameter & iCT & iCT + KNN & iCT + STF \\
    \hline
        Training iterations & 400,000 & 400,000 & 400,000 \\
        Batch Size & 1024 & 1024 & 1024 \\
        Learning Rate & 0.0001 & 0.0001 & 0.0001 \\
        EMA Decay Rate & 0.99993 & 0.99993 & 0.99993 \\
        Dropout & 0.3 & 0.3 & 0.3 \\
        ODE Solver & Euler & Huen & Huen \\
        Score Estimator & Posterior MC & KNN & STF \\
        Estimator batch size & 1 & 256 & 256 \\
        KNN search size & N/A & 2048 & N/A\\
        \hline
    \end{tabular}
    \label{tab:hyperparams}
\end{table}

To train the consistency models reported in \cref{tab:cifar_fid}, we follow the configuration outlined by \citet{song2023improved}. For the reader's convenience we reiterate that configuration here. We use the NCSN++ architecture \cite{song2020score} for all models.  Each model was trained using the improved consistency training algorithm, with the Pseudo-Huber loss function \cite{charbonnier1997deterministic} with $c=0.03$. We weight the consistency matching loss with the weighting function $\lambda(t) = \frac{1}{\sigma_{t} - \sigma_{t-1}}$. Noise levels are sampled from a log-normal distribution with $P_{mean}=-1.1$ and $P_{std}=2.0$. We follow an exponential discretzation schedule, starting with $s_0=10$ discretization steps, doubling every 50,000 training iterations to a final value of $s_1=1280$. We do not use exponential averaging on the teacher network, instead just applying a stopgrad to the student network weights. Hyperparameters for all runs are reported in \cref{tab:hyperparams}

We checkpoint each model every 2.5 million training images. We evaluate each checkpoint by generated 50,000 images and comparing against the training dataset using the torch-fidelity package \citep{obukhov2020torchfidelity}. We share random seeds across all evaluations. For each method, we select the checkpoint with the minimum FID to report in  \cref{tab:cifar_fid}.

\begin{table}[ht!]
    \centering
    \caption{Total training times for consistency model training}
    \begin{tabular}{l r r}
    \hline
         \textbf{Model} & \textbf{Training time} & \textbf{Training Time vs iCT}\\
    \hline
    iCT &   188.9 hrs  & 100\% \\
    iCT + KNN & 220.5 hrs & 117\% \\
    iCT + STF & 185.8 hrs & 98.3\% \\
    \hline
    \end{tabular}
    \label{tab:consistency_runtime}
\end{table}

The total training time for each of the above configurations is reported in \cref{tab:consistency_runtime}. We note that each model was trained on a different node of A100 GPUs which may explain how the iCT+STF model was able to train faster than the iCT model despite extra score estimation overhead.

%%%%%%%%%%%%%%%%%%%%%%%%%%%%%%%%%%%%%%%%%%%%%%%%%%%%%%%%%%%%%%%%%%%%%%%%%%%%%%%
%%%%%%%%%%%%%%%%%%%%%%%%%%%%%%%%%%%%%%%%%%%%%%%%%%%%%%%%%%%%%%%%%%%%%%%%%%%%%%%

\end{document}